\documentclass[uspaper,12pt]{article}

\usepackage{listings}
\usepackage{mathrsfs}
\usepackage{color}
\usepackage{amssymb}
\usepackage{afterpage}
\usepackage{amsfonts}
\usepackage{mathtools}
\usepackage{graphicx}
\usepackage{amsmath}
\usepackage{verbatim}
\usepackage{bm}
\usepackage{bbm}
\usepackage{enumerate}
\usepackage{amsthm}
\usepackage{stmaryrd}
\usepackage{tabu}
\usepackage{faktor}
\usepackage{algorithm,algcompatible}
\usepackage{geometry}
\usepackage{graphicx}
\usepackage{authblk}

\lstdefinestyle{bash}
{
    backgroundcolor=\color{black},
    basicstyle=\scriptsize\color{white}\ttfamily
}

\newcommand*\mcupinn[2]{\vcenter{\hbox{$\mathsurround=0pt
  \ifx\displaystyle#1\textstyle\else#1\fi\bigcup$}}}

\newcommand*\mcapinn[2]{\vcenter{\hbox{$\mathsurround=0pt
  \ifx\displaystyle#1\textstyle\else#1\fi\bigcap$}}}

\DeclarePairedDelimiter{\norm}{\lVert}{\rVert}
\DeclarePairedDelimiter{\parr}{(}{)}
\DeclarePairedDelimiter{\parq}{[}{]}

\DeclarePairedDelimiter{\bra}{\lbrace}{\rbrace}

\DeclarePairedDelimiter{\prodscal}{\langle}{\rangle}

\DeclareMathOperator*{\argmin}{arg\,min}

\DeclareMathOperator*{\st}{\,:\,}

\newcommand{\bx}{{\bm x}}
\newcommand{\by}{{\bm y}}

\newcommand{\bv}{{\bm v}}
\newcommand{\bu}{{\bm u}}

\newcommand{\bK}{{\bm K}}
\newcommand{\bY}{{\bm Y}}
\newcommand{\bH}{{\bm H}}
\newcommand{\bU}{{\bm U}}
\newcommand{\bal}{{\bm \alpha}}
\newcommand{\bmu}{{\bm \mu}}
\newcommand{\bz}{{\bm 0}}

\usepackage{fancyhdr}
\usepackage[svgnames]{xcolor}
    \definecolor{notecolor}{RGB}{137,89,168}
    \definecolor{quotecolor}{RGB}{66,113,174}
    \definecolor{warningcolor}{RGB}{249,145,87}
\usepackage[colorlinks]{hyperref}
	\hypersetup{linkcolor=[RGB]{200,40,41}}
	\hypersetup{citecolor=[RGB]{66,113,174}}
	\hypersetup{urlcolor=[RGB]{231,197,71}}
	
\usepackage{tikz}
    \usetikzlibrary{
        calc,
        decorations.pathmorphing,
        fadings,
        shadows,
        shapes.geometric,
        shapes.misc,
        }
\usepackage{tikz-cd}

\newcommand{\inner}[2]{\left\langle{#1},{#2}\right\rangle}

\newcommand{\grad}{\mathrm{grad}}
\newcommand{\Hess}{\mathrm{Hess}}

\newcommand{\D}{\mathrm{D}}

\newcommand{\calN}{\mathcal{N}}
\newcommand{\calM}{\mathcal{M}}

\newcommand{\calU}{\mathcal{U}}

\renewcommand{\ker}{\mathcal{K}er}

\newcommand{\RR}{\mathbb{R}}

\newcommand{\SSS}{\mathbb{S}}

\newtheorem{theorem}{Theorem}[section]

\newtheorem{proposition}{Proposition}[section]

\newcommand{\defeq}{\mathrel{\mathop:}=}

\def\*#1{\bm{#1}}

\begin{document}

\title{Learning the Kernel for Classification and Regression}
\author[1]{Chen Li}
\author[1]{Luca Venturi}
\author[1]{Ruitu Xu}
\affil[1]{Courant Institute of Mathematical Science, New York University, New York, USA} 
\date{\today}
\maketitle

\begin{abstract}
We investigate a series of learning kernel problems with polynomial combinations of base kernels, which will help us solve regression and classification problems. We also perform some numerical experiments of polynomial kernels with regression and classification tasks on different datasets.
\end{abstract}

\section{Introduction}

The study of kernel learning has spawned panoply of fascinating research in many important areas. In this project, we studied diverse methods to learn linear and polynomial combinations of kernels in regression and classification setups. We started off with the state-of-art algorithm in the \cite{cortes2009learning}.

In the first part, we consider the problem of learning the kernel for Kernel Ridge Regression. Starting from the dual formula one can derive several Gradient Descent type algorithms, depending on the family of kernels chosen and on possible regularizations. This type of algorithms was first proposed in \cite{varma2009more} in a very general setting.

Starting from the general setting, we look at different algorithms solving the learning kernel problem for the families of kernels that we consider. We analyze the Interpolated Iterative Algorithm (IIA) (proposed in \cite{cortes20092}) and the Projection-Based Gradient Descent Algorithm (PGD) (proposed in \cite{cortes2009learning}). For this second one, we furnish some more detail on its convergence (Proposition \ref{prop:geo}). We then look  to a slightly modified optimization problem and we derive a Regularized Interpolated Iterative Algorithm (rIIA), for the linear case, and a Regularized Projection-Based Gradient Descent Algorithm (rPGD2), for the polynomial case. We finally briefly discuss about the generalization error for this learning problem.

The above algorithms are then tested on several UCI datasets. We reported the results from our implementation and briefly commented them. Finally we ask ourselves how the kernel learned with the above algorithms could perform for SVM. Some empirical results are reported and discussed.

More empirical results are reported in Appendix, together with a more detailed proof of proposition \ref{prop:geo}. We also discuss some ideas from manifold optimization which could be used instead of the presented PGD algorithm.

\section{Algorithms for Kernel Learning} 

\subsection{Kernel Ridge Regression}

We consider the problem of learning the kernel for Kernel Ridge Regression (KRR). Be $S = \bra{(x_1,y_1),\dots,(x_m,y_m)}$ the training sample and $\by = [y_1,\dots,y_m]^T\in\mathbb{R}^m$ the vector of training set labels and $\Phi(x)\in \mathbb{R}^d$ the feature vector associated to an input data $x\in \mathbb{R}^n$. 
The primal formulation of the associated KRR problem reads
\begin{equation}\label{krr_primal}
\min_w \; \parq*{ \norm{w}_2^2 + \frac{C}{m}\sum_{i=1}^m(w^T\Phi(x_i)-y_i)^2 }\,.
\end{equation}
Problem (\ref{krr_primal}) can be equivalently formulated in its dual form: 
\begin{equation}\label{krr_dual}
\max_\bal \; \parq*{ 2\bal^T\by - \bal^T(\bK + \lambda I)\bal } \,.
\end{equation}
Here $\lambda = \frac{m}{C}$ and $\bK = \Phi^T \Phi$ is called the Gram matrix, where $\Phi = (\Phi(x_1),\dots,\Phi(x_m)) \in \mathbb{R}^{d\times m}$. This in particular shows that the problem can be generalized to consider Gram matrices of the form $\bK = (\mathcal{K}(x_i,x_j)_{ij})_{ij}$, where $\mathcal{K}:\mathbb{R}^n\times \mathbb{R}^n\to \mathbb{R}$ is a Positive Semi-Definite (PSD) Kernel function. The value $\mathcal{K}(x,y)$ of such a function is often interpreted as a measure of the similarity between the two points $x,y$. 
The maximum in (\ref{krr_dual}) is obtained for $\bal = (\bK + \lambda I )^{-1} \by$ and it is equal to
\begin{equation}\label{krr_maximum}
\by^T (\bK + \lambda I)^{-1}\by \,.
\end{equation}

\subsection{Learning the kernel}

Assume we now have a (parametrized) family of PSD kernel functions we can choose from:
\begin{equation*}
\mathbb{K}_\Theta = \bra{ \mathcal{K}_\bmu \st \bmu \in \Theta}\,.
\end{equation*}
The question is the following: how do we pick $\mathcal{K}\in\mathbb{K}_\Theta$ that represent our data the best? This is done by solving the following problem:
\begin{equation}\label{lk_krr_general}
\min_{\bmu\in\Theta}\, F(\bmu) \doteq \min_{\bmu\in\Theta}\,\parq*{ \by^T (\bK_\bmu + \lambda I)^{-1}\by + r(\bmu)} \,. 
\end{equation}
We denoted $\bK_\bmu = (\mathcal{K}_\bmu(x_i,x_j)_{ij})_{ij}$ the Gram matrix associated to the kernel function $\mathcal{K}_\bmu$. The function $r(\bmu)$ is an additional regularization term.
A general algorithm was proposed in \cite{varma2009more} and it basically consists of a projected gradient descent method for the optimization problem (\ref{lk_krr_general}). The pseudo-code is reported in Algorithm \ref{alg:lk_krr_general}. The formulation is justified by the fact that 
\begin{align*}
\frac{d}{d\mu_k} \by^T (\bK_\bmu + \lambda I)^{-1} \by \,=\, - \bal^T \frac{d}{d\mu_k}\bK_\bmu \bal\,,
\end{align*}
where $\bal = (\bK_\bmu + \lambda I )^{-1} \by$.

\begin{algorithm}
\caption{Generalized MKL}
\label{alg:lk_krr_general}
\begin{algorithmic}[1]
\STATE Initialize $\bmu_\mathrm{init} \in \Theta$  
\STATE $\bmu' = \bmu_\mathrm{init}$
\WHILE{$\norm{\bmu'-\bmu}\geq \epsilon$}
    \STATE $\bmu = \bmu'$
    \STATE $\bal = (\bK_\bmu +\lambda I)^{-1}\by$
    \STATE $\bH_k = \frac{d}{d\mu_k}\bK_\bmu$, for $k\in[1,p]$
    \STATE $\mu_k' = \mu_k - \eta\parq*{\frac{d}{d\mu_k}r(\bmu)- \,\bal^T\bH_k\bal}$, for $k\in[1,p]$
    \STATE Project $\bmu'$ on $\Theta$
\ENDWHILE
\end{algorithmic}
\end{algorithm}

\subsubsection{Linear combinations}

Suppose $\mathcal{K}_i$ are PSD kernel functions, for $i=1,\dots,p$. Then a natural family of kernels than one can consider is
\begin{equation}
\mathbb{K}_l  \, = \, \bra*{ \sum_{i=1}^p \mu_i\,\mathcal{K}_i \st \bmu \in \mathcal{M}}
\end{equation}
Here $\mathcal{M}$ is some convex subset of $[0,\infty)^n$.
This family of kernels has been widely studied, both for KRR and Support Vector Machines (SVM) optimization problems (see e.g. \cite{lanckriet2004learning}). In our work, we look at three specific algorithms for this problem.

\paragraph{Projection-Based Gradient Descent Algorithm (PGD)}

This algorithm is simply Algorithm \ref{alg:lk_krr_general} for $r = 0$ and the family $\mathbb{K}_l$. The partial derivatives $\bH_k$ in this case are given by
\begin{equation*}
\bH_k = \bK_k, \quad \text{where} \quad  \bK_k = (\mathcal{K}_k(x_i,x_j)_{ij})_{ij},
\end{equation*}
for $k\in[1,p]$. The parameter space is chosen of the form
\begin{equation*}
\mathcal{M} = \bra{ \bmu \geq \bz \st \norm{\bmu - \bmu_0}_q \leq \Lambda}.
\end{equation*}
Here $q\geq 1$, $\Lambda > 0$ and $\bmu_0\geq \bz$ are parameters defining the set $\mathcal{M}$. A typical choice for $\bmu_0$ is usually $\mathbf{1}$ or $\mathbf{0}$. The pseudo-code of this algorithm is reported in Algorithm \ref{alg:lk_krr_pgd1}, which was first analyzed in \cite{cortes2009learning}.

\begin{algorithm}
\caption{Linear PGD}
\label{alg:lk_krr_pgd1}
\begin{algorithmic}[1]
\STATE Initialize $\bmu_\mathrm{init} \in \mathcal{M}$  
\STATE $\bmu' = \bmu_\mathrm{init}$
\WHILE{$\norm{\bmu'-\bmu}\geq \epsilon$}
    \STATE $\bmu = \bmu'$
    \STATE $\bal = (\bK_\bmu +\lambda I)^{-1}\by$
    \STATE $\mu_k' = \mu_k + \eta\,\bal^T\bK_k\bal$, for $k\in[1,p]$
    \STATE Normalize $\bmu'$ s.t. $\norm{\bmu'-\bmu_0}_q =\Lambda$
\ENDWHILE
\end{algorithmic}
\end{algorithm}

\paragraph{Interpolated Iterative Algorithm (IIA)}

This algorithm is a modification of Algorithm \ref{alg:lk_krr_pgd1} for the case $q=2$. It is based on exploiting the particular structure of the solution to the optimization  problem
\begin{equation}\label{lk_krr_linear}
\min_{\bmu\in\mathcal{M}} \,\parq*{\by^T \parr[\Big]{ \sum_{k=1}^p \mu_k \bK_k + \lambda I }^{-1} \by}
\end{equation}
In \cite{cortes20092} it was proved the following:
\begin{theorem}
The solution $\bmu$ to the optimization problem (\ref{lk_krr_linear}) is given by $\bmu = \bmu_0 + \Lambda \frac{\bv}{\norm{\bv}}$ with $\bv = (v_1,\dots,v_p)$ given by $v_k = \bal^T \bK_k \bal$.
\end{theorem}
In they same work, the authors propose an algorithm based on the above result. The pseudo-code of this algorithm is reported in Algorithm \ref{alg:lk_krr_iia}. 

\begin{algorithm}
\caption{IIA}
\label{alg:lk_krr_iia}
\begin{algorithmic}[1]
\STATE Initialize $\bmu_\mathrm{init}\in\mathcal{M}$
\STATE $\bal' = (\bK_{\bmu_\mathrm{init}} +\lambda I)^{-1}\by$
\WHILE{$\norm{\bal'-\bal}\geq \epsilon$}
    \STATE $\bal = \bal'$
    \STATE $ v_k = \bal^T\bK_k\bal$, for $k\in[1,p]$
    \STATE $\bmu = \bmu_0 + \Lambda \frac{\bv}{\norm{\bv}}$
    \STATE $\bal' = \eta\bal + (1-\eta)(\bK_\bmu + \lambda I)^{-1}\by$
\ENDWHILE
\end{algorithmic}
\end{algorithm}

\paragraph{Regularized Interpolated Iterative Algorithm (rIIA)}

The last algorithm we analyze for the problem of learning a linear combination of kernels is based on substituting the feasibility condition $\norm{\bmu - \bmu_0} = \Lambda$ with a regularization term. In this case $\mathcal{M} = \bra{\bmu \geq \bz}$ and we aim to minimize the function 
\begin{equation}\label{regularized_loss}
F(\bmu) = \by^T \parr[\Big]{ \sum_{k=1}^p \mu_k \bK_k + \lambda I }^{-1} \by + \beta \norm{\bmu}^2.
\end{equation}
Here $\beta > 0$ is a regularization parameter. Instead of writing the Generalized MKL algorithm for this case, we look at the special structure of our problem. If we compute the gradient and the Hessian of $F$ we get
\begin{align*}
\partial_k F(\bmu) & = -\bal^T \bK_k \bal + 2\beta \mu_k, \\
\partial^2_{jk} F(\bmu) & = \bal^T\bK_k (\bK_\bmu + \lambda I)^{-1} \bK_j\, \bal + 2\beta \,\delta_{jk},
\end{align*}
where $\delta_{jk}$ denotes the Kronecker delta. In particular the function $F(\bmu)$ is convex, since for all $\bu \in\mathbb{R}^p$ it holds
\begin{equation*}
\bu^T \nabla^2 F(\bmu) \,\bu =  \bal^T\bK_\bmu (\bK_\bmu + \lambda I)^{-1} \bK_\bmu\, \bal + 2\beta \,\norm{\bu}^2 \geq 0\,.
\end{equation*}
Therefore, the global minima of $F$ is obtained at any stationary point. Since the form of the gradient implies that any stationary point is such that $\bmu \geq \bz$, the following holds:
\begin{theorem}
The minima of $F$ over $\mathcal{M}=\bra{\bmu\geq \bz}$ is obtained at $\bmu$ satisfying
\begin{equation*}
\mu_k = \frac{1}{2\beta} \bal^T \bK_k \bal \quad\text{for}\quad k \in [1,p].
\end{equation*}
\end{theorem}
The above theorems motivates the following iterative interpolation algorithm, whose pseudo-code is reported in Algorithm \ref{alg:lk_krr_riia}.

\begin{algorithm}
\caption{rIIA}
\label{alg:lk_krr_riia}
\begin{algorithmic}[1]
\STATE Initialize $\bmu_\mathrm{init}\in\mathcal{M}$
\STATE $\bal' = (\bK_{\bmu_\mathrm{init}} +\lambda I)^{-1}\by$
\WHILE{$\norm{\bal'-\bal}\geq \epsilon$}
    \STATE $\bal = \bal'$
    \STATE $\mu_k = \frac{1}{2\beta}\,\bal^T\bK_k\bal$, for $k\in[1,p]$
    \STATE $\bal' = \eta\bal + (1-\eta)(\bK_\bmu + \lambda I)^{-1}\by$
\ENDWHILE
\end{algorithmic}
\end{algorithm}

\subsubsection{Polynomial combinations}

We now consider the family of polynomial combinations of $\mathcal{K}_i$, $i\in[1,p]$. In the most general form this family is described by
\begin{equation*}
\mathbb{K}_p = \bra*{ \sum_{\stackrel{k_1,\dots,k_p\geq 0}{k_1+\cdots + k_p \leq d}} \mu_{k_1\cdots k_p}\,\mathcal{K}_1^{k_1}\cdots \mathcal{K}_p^{k_p} \st \bmu \in \mathcal{M}}\,.
\end{equation*}
More specifically we consider the case where the coefficients $\mu_{k_1\cdots k_p}$ can be written as a product of non-negative coefficients $\mu_{k_1\cdots k_p} = \mu_1^{k_1}\cdots\mu_p^{k_p}$. The algorithms reported below are for the case $d=2$ for the ease of the presentation, even if they easily generalize. This means that we consider here
\begin{equation*}
\mathbb{K}_p = \bra*{ \sum_{i,j=1}^p \mu_i\mu_j\,\mathcal{K}_i\mathcal{K}_j \st \bmu \in \mathcal{M}}\,.
\end{equation*}
In the sequel, we will denote $\bK_\bmu = (\mathcal{K}_\bmu(x_i,x_j)_{ij})_{ij}$ and $\bK_\bmu^{\circ 2} = \bK_\bmu \circ \bK_\bmu$.

\paragraph{Projection-Based Gradient Descent Algorithm (PGD2)}

This algorithm is the generalization of Algorithm \ref{alg:lk_krr_pgd1} to the polynomial setting. In this case the gradient and the Hessian of $F$ read
\begin{align*}
\partial_k F(\bmu) & = -2 \bal^T \bU_k \bal, \quad \text{for}\quad \bU_k = \parr[\Big]{\sum_{j=1}^p\mu_j\bK_j} \circ \bK_k, \\ 
\partial^2_{jk} F(\bmu) & = 8\,\bal^T\bU_k (\bK_\bmu^{\circ 2} + \lambda I)^{-1} \bU_j\, \bal - 2\,\bal^T\bK_j \circ\bK_k\, \bal.
\end{align*}
Based on the expression of the gradient the algorithm is written by plugging this expression in the Generalized MKL. A pseudo-code for this algorithm is reported in Algorithm \ref{alg:lk_krr_pgd2}, which was originally proposed in \cite{cortes2009learning}. 

\begin{algorithm}
\caption{Quadratic PGD}
\label{alg:lk_krr_pgd2}
\begin{algorithmic}[1]
\STATE Initialize $\bmu_\mathrm{init} \in \mathcal{M}$  
\STATE $\bmu' = \bmu_\mathrm{init}$
\WHILE{$\norm{\bmu'-\bmu}\geq \epsilon$}
    \STATE $\bmu = \bmu'$
    \STATE $\bal = (\bK_\bmu +\lambda I)^{-1}\by$
    \STATE $\mu_k' = \mu_k + 2\,\eta\,\bal^T\bU_k\bal$, for $k\in[1,p]$
    \STATE Normalize $\bmu'$ s.t. $\norm{\bmu'-\bmu_0}_q =\Lambda$
\ENDWHILE
\end{algorithmic}
\end{algorithm}

In the same paper is presented also a convergence analysis for this method. It is based on the following results.
\begin{proposition}
Any stationary point $\bmu^*$ of $F$ necessarily maximizes $F$.
\end{proposition}
\begin{proposition}
If any point $\bmu^*>\bz$ is a stationary point of $F$, then the function is necessarily constant.
\end{proposition}
These prepositions are sufficient to show that the gradient descent algorithm will not become stuck at a local minima while searching the interior of the convex set $\mathcal{M}$ and, furthermore, they indicate that the optimum is found at the boundary. Then, in the paper, a necessary and sufficient condition for the convexity of $F$ on $\mathcal{M}$ is given. Nevertheless, such condition seems quite cryptic to us. Also, they report empirical evidence of convexity of the function for small values of $\lambda$ and concavity for high values of $\lambda$. Here we give a proof of this fact.
\begin{proposition}\label{prop:geo}
The function $F$ is convex over the region $\mathcal{M}$ for sufficiently small values of $\lambda$, and it is concave for sufficiently large values of $\lambda$.
\end{proposition}
\begin{proof}
First we focus on the concavity condition. We want to show that for $\lambda$ sufficiently large, it holds
\begin{equation*}
4\,\bal^T(\bK_\bmu\circ\bK_\bu) (\bK_\bmu^{\circ 2} + \lambda I)^{-1} (\bK_\bmu\circ\bK_\bu)\, \bal \leq \,\bal^T\bK_\bu^{\circ 2}\,\bal\,,
\end{equation*}
for all $\bu$ s.t. $\norm{\bu}=1$. The LHS can be upper bounded as 
\begin{equation*}
4\,\bal^T(\bK_\bmu\circ\bK_\bu) (\bK_\bmu^{\circ 2} + \lambda I)^{-1} (\bK_\bmu\circ\bK_\bu)\, \bal \leq \frac{4\norm{\bal}^2}{\lambda}C\,,
\end{equation*}
where $C = \max_{\bu \st\norm{\bu} = 1,\,\bmu \in \mathcal{M}}\norm{\bK_\bmu \circ \bK_\bu}^2$, while the RHS can be lower bounded as 
\begin{equation*}
\bal^T\bK_\bu^{\circ 2}\,\bal \geq D\norm{\bal}^2,
\end{equation*}
where $D = \min_{\bu \st \norm{\bu}=1} \lambda_\mathrm{min}(\bK_\bu^{\circ 2})$. If $D>0$, then the function $F$ is concave over $\mathcal{M}$ if 
$$
\lambda \geq \frac{4C}{D}\,.
$$
In a similar fashion we can prove the convexity condition. We want to show that for $\lambda$ sufficiently small, it holds 
\begin{equation*}
4\,\bal^T(\bK_\bmu\circ\bK_\bu) (\bK_\bmu^{\circ 2} + \lambda I)^{-1} (\bK_\bmu\circ\bK_\bu)\, \bal \geq \,\bal^T\bK_\bu^{\circ 2}\,\bal\,,
\end{equation*}
for all $\bu$ s.t. $\norm{\bu}=1$.
The LHS can be lower bounded as 
\begin{equation*}
4\,\bal^T(\bK_\bmu\circ\bK_\bu) (\bK_\bmu^{\circ 2} + \lambda I)^{-1} (\bK_\bmu\circ\bK_\bu)\, \bal \geq \frac{4\norm{\bal}^2}{HE}e^{-\lambda / E}\,,
\end{equation*}
where $E = \max_{\bmu \in\mathcal{M}} \lambda_\mathrm{max}(\bK_\bu^{\circ 2})$ and $H = \max_{\bu \st\norm{\bu} = 1,\,\bmu \in \mathcal{M}}\norm{(\bK_\bmu \circ \bK_\bu)^{-1}}^2$, while the RHS can be upper bounded as 
\begin{equation*}
\bal^T\bK_\bu^{\circ 2}\,\bal \leq B\norm{\bal}^2,
\end{equation*}
where $B = \max_{\bu \st \norm{\bu}=1} \lambda_\mathrm{max}(\bK_\bu^{\circ 2})$. Then the function $F$ is convex over $\mathcal{M}$ if 
$$
\lambda \leq E\log \frac{4}{EHB}\,.
$$
In Appendix \ref{app:proof_} we report a more detailed proof of these bounds.
\end{proof}

\paragraph{Regularized Projection-Based Gradient Descent Algorithm (rPGD2)}

The last algorithm we propose aims to minimize the function (\ref{regularized_loss}) for polynomial combinations of kernels. We could think to define such an algorithm by miming Algorithm \ref{alg:lk_krr_riia}. Unfortunately, in this case, such an algorithm would not be guaranteed to converge. This is due to the fact the function $F$ is not convex anymore, and therefore it is not guaranteed to attain a minimum at a stationary point. Therefore, we just write the Generalized MKL for this setting. The pseudo-code is reported in Algorithm \ref{alg:lk_krr_rpgd2}.

\begin{algorithm}
\caption{Regularized Quadratic PGD}
\label{alg:lk_krr_rpgd2}
\begin{algorithmic}[1]
\STATE Initialize $\bmu_\mathrm{init} \in \mathcal{M}$  
\STATE $\bmu' = \bmu_\mathrm{init}$
\WHILE{$\norm{\bmu'-\bmu}\geq \epsilon$}
    \STATE $\bmu = \bmu'$
    \STATE $\bal = (\bK_\bmu +\lambda I)^{-1}\by$
    \STATE $\mu_k' = \mu_k + 2\,\eta\,(\beta \mu_k-\,\bal^T\bU_k\bal)$, for $k\in[1,p]$
    \STATE $\mu_k' = \max\bra{\mu_k', 0}$, for $k\in[1,p]$
\ENDWHILE
\end{algorithmic}
\end{algorithm}

\subsubsection{Generalization error for kernel learning}

A natural question is why should polynomial combinations of kernels work better than linear.  
For this project, we also tried to give a bound on the generalization bound for learning polynomial combination of kernel. In \cite{cortes2010generalization}, the following theorems was proved for the generalization error of learning linear combinations.
\begin{theorem}
If $\mathcal{H}_l$ is the family of functions
\begin{equation*}
\mathcal{H}_l = \bra*{ \bx \mapsto w^T\Phi_\mathcal{K}(\bx) \st  \mathcal{K} = \sum_{k=1}^p \mu_k \mathcal{K}_k,\, \bmu\geq \bz, \, \norm{\bmu}^2_2=1 }\,,
\end{equation*}
and $R>0$ is such that $\mathcal{K}_k(x,x)\leq R^2$ for all $x\in\mathbb{R}^n$ and $k\in[1,p]$, then the Rademacher complexity of $\mathcal{H}_l$ (for any sample set $S$ of size $m$) can be bounded as 
\begin{equation*}
\hat{\mathcal{R}}_S(\mathcal{H}_l) \leq \frac{\eta_0 \, p^{1/4} \, R}{\sqrt{m}},  
\end{equation*}
where $\eta_0 = \sqrt{23/11}$ is a constant. In particular this implies the following generalization bound, for a fixed $\rho>0$ and $\delta\in(0,1)$, with probability at least $1-\delta$:
\begin{equation*}
R(h) \leq \hat{R}_\rho(h) + \frac{\eta_1 \, p^{1/4} \, R}{\rho\sqrt{m}} + 3\sqrt{\frac{\log\frac{2}{\delta}}{2m}},
\end{equation*}
where $\eta_1 = \sqrt{46/11}$ is a constant.
\end{theorem}
We would be then interested in giving a similar bound for the family of polynomial combinations of kernels that we considered:
\begin{equation*}
\mathcal{H}_p = \bra*{ \bx \mapsto w^T\Phi_\mathcal{K}(\bx) \st  \mathcal{K} = \parr[\Big]{ \sum_{k=1}^p \mu_k \mathcal{K}_k }^2,\, \bmu\geq \bz, \, \norm{\bmu}^2_2=1 }\,.
\end{equation*}
A bound for the Rademacher complexity can be obtained (with the previous hypothesis), by noticing that we can embed $\mathcal{H}_p$ in the bigger family:
\begin{equation*}
\mathcal{H}_p\subset \mathcal{H}_p^+  \doteq \bra*{ \bx \mapsto w^T\Phi_\mathcal{K}(\bx) \st  \mathcal{K} = \sum_{j,k=1}^p \mu_{jk} \mathcal{K}_j\mathcal{K}_k ,\, \bmu\geq \bz, \, \norm{\bmu}^2_2=1 }\,.
\end{equation*}
Since $\mathcal{K}_j(x,x)\mathcal{K}_k(x,x) \leq R^4$, the following bound holds:
\begin{align*}
\hat{\mathcal{R}}_S(\mathcal{H}_l) & \leq \frac{\eta_0 \, p^{1/2} \, R^2}{\sqrt{m}}, \\
R(h) & \leq \hat{R}_\rho(h) + \frac{\eta_1 \, p^{1/2} \, R^2}{\rho\sqrt{m}} + 3\sqrt{\frac{\log\frac{2}{\delta}}{2m}},
\end{align*}
where the last one holds for every $h\in\mathcal{H}_p$, with probability at least $1-\delta$.
Unfortunately, such bounds are worse then the ones for linear combinations of kernel. We tried to 
provide better bounds, but none of the techniques we tried reached bounds as good as the above.

\subsection{Empirical Results}

\subsubsection{Data}

To empirically test our algorithms we considered several different datasets. These datasets were obtained from the UCI Machine Learning Repository. A brief description of some of the datasets we used is reported below.

\begin{itemize}
\item {\bf Breast Cancer Data Set.} This breast cancer domain was obtained from the University Medical Centre,
   Institute of Oncology, Ljubljana, Yugoslavia. This data set includes 201 instances of one class and 85 instances of
     another class.  The instances are described by 9 attributes, some of
     which are linear and some are nominal.
     
\item {\bf Diabetes Data Set.} This data set contains the distribution for 70 sets of data recorded
on diabetes patients (several weeks' to months' worth of glucose, insulin,
and lifestyle data per patient + a description of the problem domain).

\item {\bf Ionosphere Data Set.} This radar data was collected by a system in Goose Bay, Labrador.  This
   system consists of a phased array of 16 high-frequency antennas with a
   total transmitted power on the order of 6.4 kilowatts.  See the paper
   for more details.  The targets were free electrons in the ionosphere.
   "Good" radar returns are those showing evidence of some type of structure 
   in the ionosphere.  "Bad" returns are those that do not; their signals pass
   through the ionosphere.  

   Received signals were processed using an autocorrelation function whose
   arguments are the time of a pulse and the pulse number.  There were 17
   pulse numbers for the Goose Bay system.  Instances in this databse are
   described by 2 attributes per pulse number, corresponding to the complex
   values returned by the function resulting from the complex electromagnetic
   signal. 
   
It has 351 instances and 34 attributes.

\item {\bf Heart Disease Data Set.} This database contains 76 attributes, but all published experiments
     refer to using a subset of 14 of them.  In particular, the Cleveland
     database is the only one that has been used by ML researchers to 
     this date.  The "goal" field refers to the presence of heart disease
     in the patient.  It is integer valued from 0 (no presence) to 4.
     Experiments with the Cleveland database have concentrated on simply
     attempting to distinguish presence (values 1,2,3,4) from absence (value
     0). 
     
\item {\bf Connectionist Bench (Sonar, Mines vs. Rocks) Data Set.} This is the data set used by Gorman and Sejnowski in their study
of the classification of sonar signals using a neural network.  The
task is to train a network to discriminate between sonar signals bounced
off a metal cylinder and those bounced off a roughly cylindrical rock. 

Each pattern is a set of 60 numbers in the range 0.0 to 1.0.  Each number
represents the energy within a particular frequency band, integrated over
a certain period of time.  The integration aperture for higher frequencies
occur later in time, since these frequencies are transmitted later during
the chirp.


\end{itemize}


\subsubsection{Our implementation}

We tried to implement the above algorithm in \texttt{sklearn}. All the code we wrote is contained in the attached \texttt{.zip} file. We wrote a class \texttt{problem} where a problem is defined for a dataset together with a kernel learning algorithm and all the required parameters. In the tables below we reported some results for the \texttt{Ionosphere} ad \texttt{Sonar} datasets. First we run 10-Fold Cross Validation to select the best parameters for each method. The test error
reported was based on 30 random 50/50 splits of the entire dataset into training and test sets. The types of error reported are the square root of the mean square error (MSE) and the misclassification fraction (MSF). The labels were recovered from the regression output by simply applying the sign function. The parameter $\eta$ was chosen as $1$ for PGD-type algorithms (being reduced by a $0.8$ factor if the error increased) and as $1/2$ for IIA-type algorithms. The number of the iterations for each number were between $10$ and $50$ depending on the parameters and the algorithm. We also compared the algorithm with a benchmark model (BM) and with a uniform combinations of the kernels (UNIF). In the following $d$ denotes the degree of the combinations learned. 

\setlength\tabcolsep{1.5pt} 
\begin{table}[p]
\centering
{\renewcommand\arraystretch{1.2}
\begin{tabular}{|l|c|c|c|c|}
\hline
 & PGD & IIA & rPGD & rIIA \\
\hline
$d = 1$ & $0.69\pm 0.10$ & $0.69\pm 0.10$ & $0.70 \pm 0.11$ & $0.70 \pm 0.11$ \\
\hline
$d = 2$ & $0.75\pm 0.13$ & $-$ & $0.76 \pm 0.15$ & $-$ \\
\hline
\end{tabular}}
\caption{Cross validation results for regression on dataset \texttt{Ionosphere}.}
\end{table}
\begin{table}[p]
\centering
{\renewcommand\arraystretch{1.2}
\begin{tabular}{|l|c|c|c|c|c|c|}
\hline
 & IIA & PGD $(d=2)$ & rIIA & rPGD $(d=2)$ & BM & UNIF \\
\hline
MSE & $0.82\pm 0.03$ & $0.86\pm 0.02$ & $0.83 \pm 0.02$ & $0.94 \pm 0.01$ & $0.82 \pm 0.03$ & $1.11 \pm 0.16$ \\
\hline
MSF & $0.23\pm 0.03$ & $0.24\pm 0.04$ & $0.24 \pm 0.02$ & $0.29 \pm 0.05$ & $0.23 \pm 0.04$ & $0.29 \pm 0.06$ \\
\hline
\end{tabular}}
\caption{Mean of $30$ ($50/50$) test results for regression on dataset \texttt{Ionosphere}.}
\end{table}


\begin{table}[p]
\centering
{\renewcommand\arraystretch{1.2}
\begin{tabular}{|l|c|c|c|c|}
\hline
 & PGD & IIA & rPGD & rIIA \\
\hline
$d = 1$ & $0.80\pm 0.15$ & $0.80\pm 0.15$ & $0.82 \pm 0.18$ & $0.82 \pm 0.18$ \\
\hline
$d = 2$ & $0.80\pm 0.11$ & $-$ & $0.81 \pm 0.13$ & $0.81 \pm 0.13$ \\
\hline
\end{tabular}}
\caption{Cross validation results for regression on dataset \texttt{Sonar}.}
\end{table}
\begin{table}[p]
\centering
{\renewcommand\arraystretch{1.2}
\begin{tabular}{|l|c|c|c|c|c|c|}
\hline
 & IIA & PGD $(d=2)$ & rIIA & rPGD $(d=2)$ & BM & UNIF \\
\hline
MSE & $0.82\pm 0.03$ & $0.86\pm 0.04$ & $0.83 \pm 0.02$ & $0.86 \pm 0.04$ & $1.11 \pm 0.16$ & $1.69 \pm 0.02$ \\
\hline
MSF & $0.22\pm 0.03$ & $0.26\pm 0.03$ & $0.23 \pm 0.03$ & $0.25 \pm 0.03$ & $0.22 \pm 0.03$ & $0.25 \pm 0.03$ \\
\hline
\end{tabular}}
\caption{Mean of $30$ ($50/50$) test results for regression on dataset \texttt{Sonar}.}
\end{table}

Both datasets and CV and test errors show the same trend. First of all we notice that for a given degree of the combination the performances of different algorithms are almost the same, even if the rPGD method seems to be more unstable for $d > 1$. The linear combination learned is much better then the uniform combination but performs as well as the benchmark. Instead, polynomial combinations seem to perform worse than linear ones. In the following we reported the plots of the (CV and test) error as a function of the regularization parameter $\lambda$ for different datasets to try to understand why this is the case. The plots are for different values of the degree $d$. 



\begin{figure}[p]
\includegraphics[width=\textwidth]{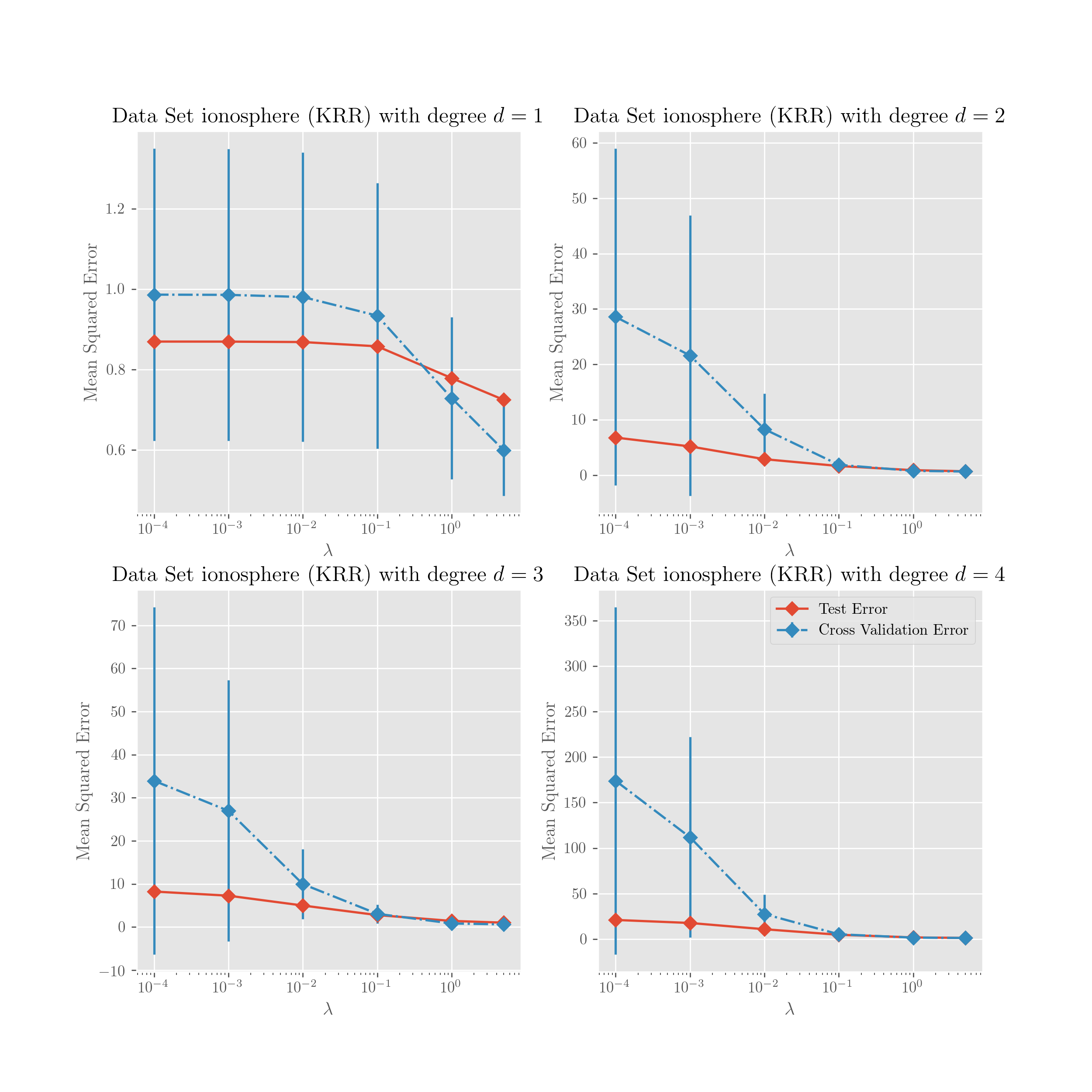}
\caption{PGD results for regression on dataset \texttt{Ionosphere}.}
\end{figure}


\begin{figure}[p]
\includegraphics[width=\textwidth]{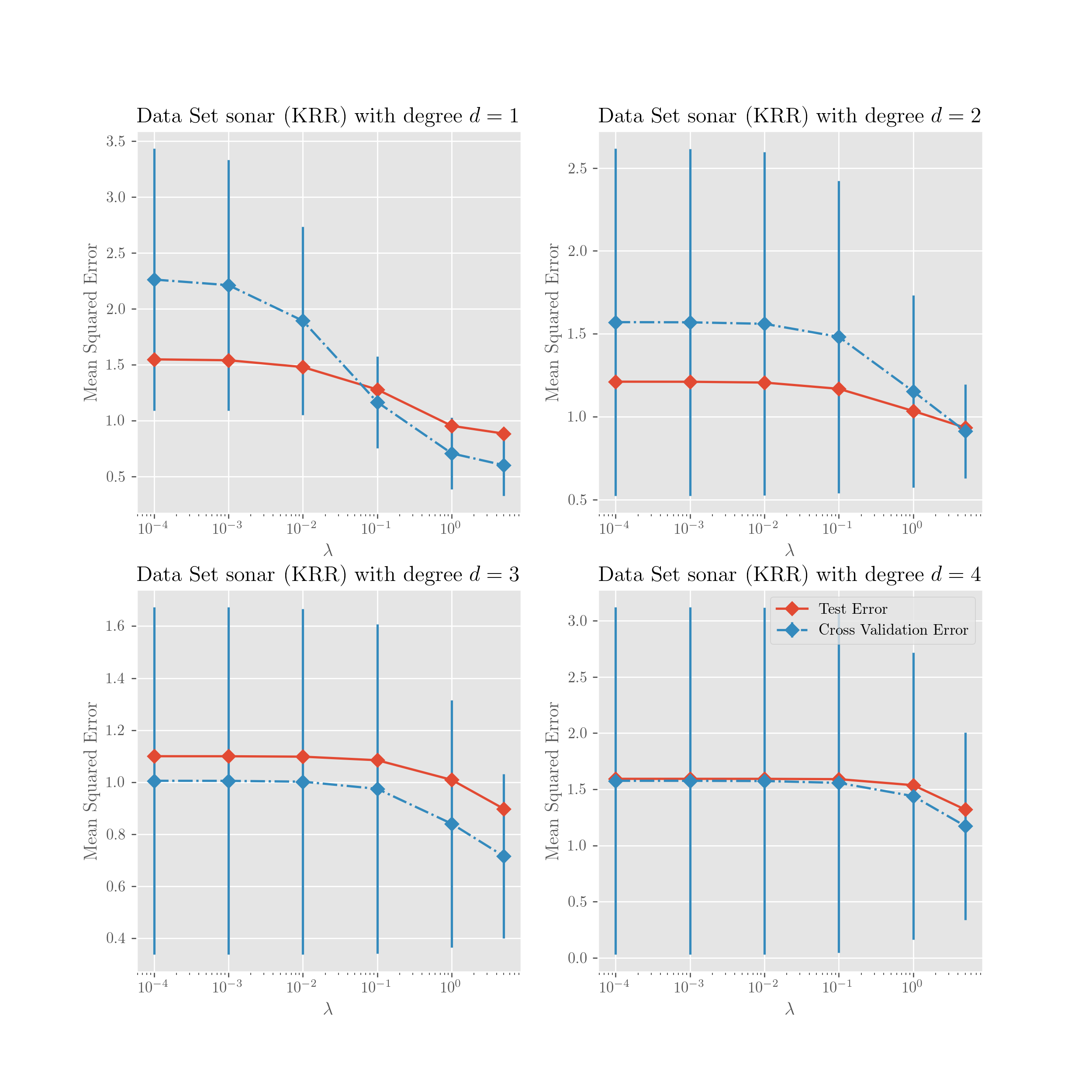}
\caption{PGD results for regression on dataset \texttt{Sonar}.}
\end{figure}


In the above plots the trend is almost the same. In both the linear and the polynomial case, the error decreases as $\lambda$ increases from $10^{-4}$ to $10$. Although the error is generally higher for $d > 1$. Notice that the parameter $\lambda$ is the same in both the KKR model used to fit the data and the learning kernel algorithm. For $d>1$, the PGD algorithm we considered is guaranteed to converge to a global minima only for small values of $\lambda$. Although for such values of the regularization parameter the KRR could furnish a poor model to fit our data. There is therefore a clear tradeoff between the two problems in the choice of $\lambda$. 

Moreover, in our tests, the algorithm showed to be extremely susceptible to the choice of other parameters such as $\bmu_0$ and $\bmu_\mathrm{init}$. In our test we set them as $\mathbf{1}$. Although other tests showed that the results were quite different for different choices. The cross validation we run was done to select the parameter $\Gamma$ in the range $[0.001,0.1,0.5,1.,2.,10.]$. We believe that a more accurate selection of the parameters could bring better results, more aligned with the ones of paper \cite{cortes2009learning}. Although, this shows that such algorithms are quite susceptible to the choice of `bad' parameters.

\newpage
\section{Classification Tasks}

We are now interested in the problem of classification. In particular, we consider the case when the labels are values in $\bra{-1,1}$. The datasets we considered before actually fall in this category. One possible approach (which we used to evaluate the misclassification rate on the test set) is to run a regression algorithm (KRR in our case) and then map the output to a corresponding label by simply using the function $x\mapsto \mathrm{sign}(x)$. In particular we are interested in the problem of learning the kernel. Another algorithm which is kernealizable it the Support Vector Machine (SVM) method. The SVM algorithm (with 2-norm soft margin, see \cite{lanckriet2004learning}) solves the optimization problem:
\begin{align*}
\min_{w,b,\xi} & \quad \parq*{ \norm{w}^2 + C\sum_{i=1}^m \xi_i^2 } \\
\text{subject to} & \quad y_i(\prodscal{w, \Phi(x_i)} + b) \geq 1 - \xi_i, \quad \text{for }i=1,\dots,n, \\ 
 & \quad \xi_i\geq 0, \quad \text{for }i=1,\dots,n \,.
\end{align*}
The same problem can be formulated in a dual kernealized version:
\begin{align}
\max_\bal & \quad \parq*{ \, 2\prodscal{\bal,\mathbf{1}} - \bal^T \parr*{ G(\bK) + \lambda I } \bal \,} \label{svm_dual} \\
\text{subject to} & \quad \bal \geq \bz,\; \prodscal{\bal, \by} = 0 \, , \notag
\end{align}
where $G(\bK) = \bY \bK \bY $, for $\bY = \mathrm{diag}(\by)$, and $\lambda = 1/C$. Notice that in this case we do not have a \emph{closed formula} for the solution $\bal$ of the dual optimization problem. Although, it is still possible to formulate a Generalized MKL algorithm (see \cite{varma2009more}). Such algorithm is reported in Algorithm \ref{alg:lk_svm_general}. Notice that, because of the absence of a closed formula, we need to solve an $SVM$ problem at each iteration, which could become very costly if the number of iterations gets large. Instead of doing this, we are interested in understanding if the kernels learned for the KKR optimization problem could fit well for SVM. We first show how the two optimization problems (\ref{krr_dual}) and (\ref{svm_dual}) can be somehow related and we then discuss some empirical results.  
Problem (\ref{svm_dual}) can be re-written as (taking $\bv = \bY \bal $ and by noticing that $\bY ^2 = I$):
\begin{align*}
\max_\bal & \quad \parq*{ \, 2\prodscal{\bv,\by} - \bv^T \parr*{ \bK + \lambda I } \bv \,} \label{svm_dual} \\
\text{subject to} & \quad \bY\bv \geq \bz,\; \prodscal{\bv, \mathbf{1}} = 0    
\end{align*}
It's therefore clear that (\ref{krr_dual}) is a relaxation of the above and therefore we can upper bound the quantity above with $\by^T (\bK + \lambda I)^{-1} \by$. Therefore solving the kernel learning problem for KRR is the same as solving the kernel learning problem on a upper bound of (\ref{svm_dual}). A natural question is how strict/loose this bound is. This is clearly related on how good a kernel, learned for KRR, could perform for SVM.  We explore this question empirically in the next section.

\begin{algorithm}
\caption{SVM Generalized MKL}
\label{alg:lk_svm_general}
\begin{algorithmic}[1]
\STATE Initialize $\bmu_\mathrm{init} \in \Theta$  
\STATE $\bmu' = \bmu_\mathrm{init}$
\WHILE{$\norm{\bmu'-\bmu}\geq \epsilon$}
    \STATE $\bmu = \bmu'$
    \STATE Solve problem (\ref{svm_dual}) with $G = G(\bK_\bmu)$ to get a new $\bal$
    \STATE $\bH_k = \frac{d}{d\mu_k}\bK_\bmu$, for $k\in[1,p]$
    \STATE $\mu_k' = \mu_k - \eta\parq*{\frac{d}{d\mu_k}r(\bmu)- \,\bal^T\bH_k\bal}$, for $k\in[1,p]$
    \STATE Project $\bmu'$ on $\Theta$
\ENDWHILE
\end{algorithmic}
\end{algorithm}

\subsection{Empirical results}

The following figures show the performances of learned kernels from PGD and IIA with different $\lambda$ on SVM given the best parameters after grid search on the datasets \texttt{Ionosphere} and \texttt{Sonar}. The plots are for different degrees $d$ of the polynomial combinations used. More plots are reported in the Appendix.

\begin{figure}[p]
\includegraphics[width=\textwidth]{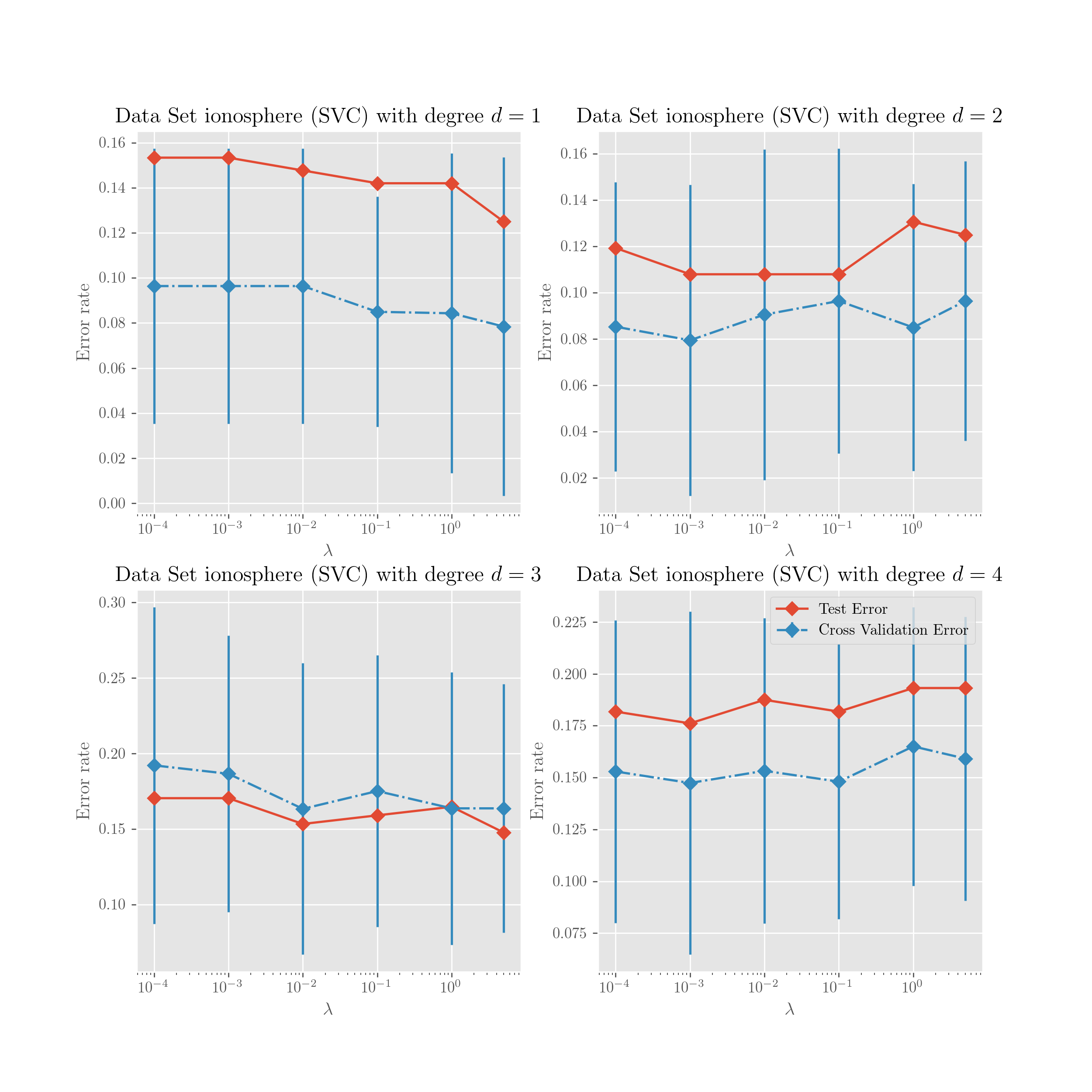}
\caption{PGD results for classification on dataset \texttt{Ionosphere}.}
\end{figure}


\begin{figure}[p]
\includegraphics[width=\textwidth]{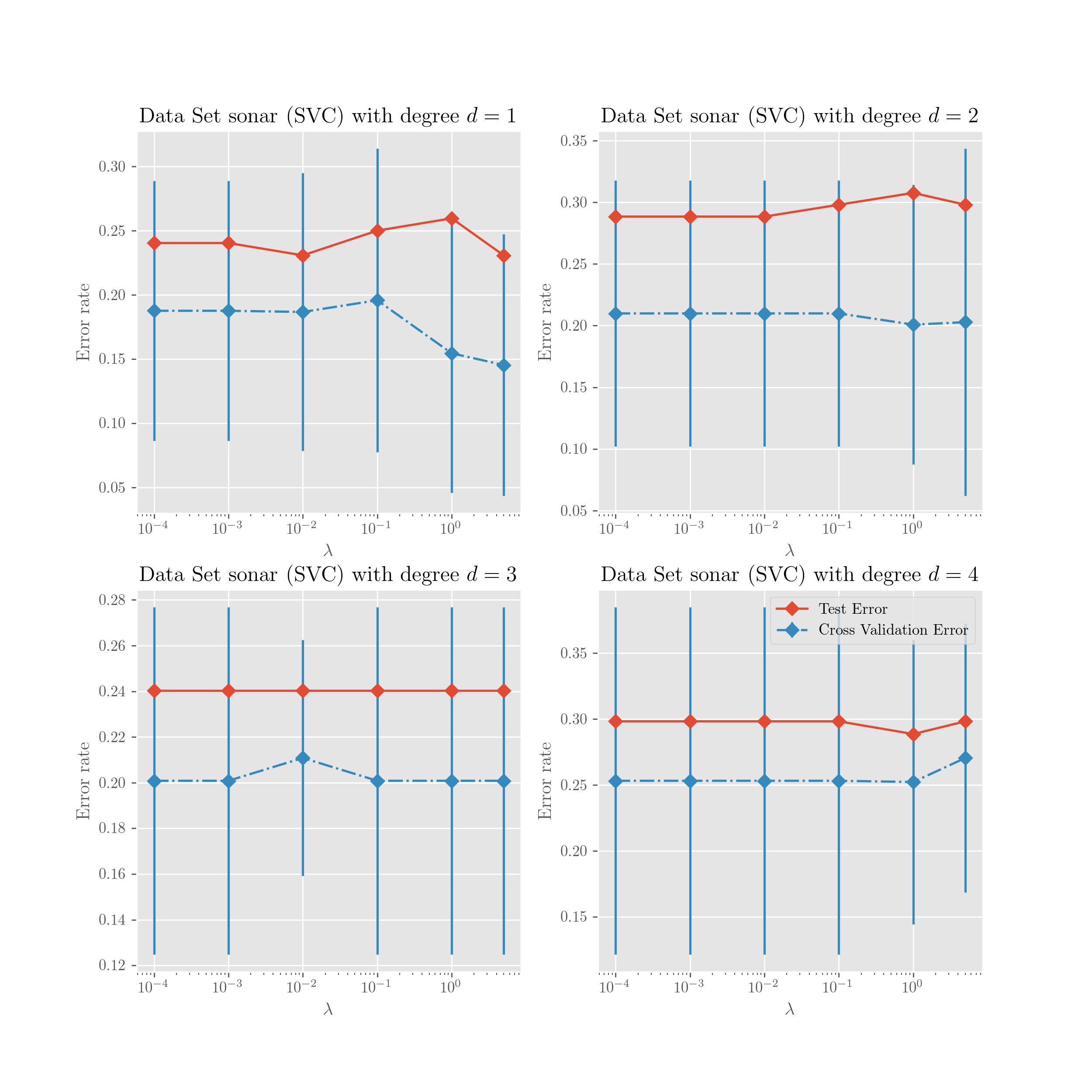}
\caption{PGD results for classification on dataset \texttt{Sonar}.}
\end{figure}


The results obtained confirm the ones obtained for KRR. Generally speaking, higher degree polynomial combinations are more difficult to learn. Nevertheless we can observe that the test errors obtained look a bit better than the ones obtained with KRR. This may be because we are using a model which is designed for classification tasks (rather than regression). Moreover for some dataset some polynomial combinations actually show an improvement on the performances. The trade-off effect we discussed in the previous section is therefore less present in this situation.
To summarize, this shows that this method could be of practical interest and it deserves further research. Final considerations are reported in the conclusion.

\newpage
\section{Conclusion} 

We considered the problem of learning the kernel for Kernel Ridge Regression. We started by reporting the general setup for this problem (as in \cite{varma2009more}) and we focused on the case of linear and polynomial combinations. In particular we considered two algorithms already proposed in \cite{cortes20092} and \cite{cortes2009learning} respectively. For the second one, we proved a more readable condition under which the gradient descent algorithm for the associated optimization problem is guaranteed to converge to a global minima. We also presented (in the appendix) ideas from manifold optimization which could perhaps be a good substitute to the PGD algorithm.

Then we considered an algorithm that we derived by considering a very close related optimization problem. The algorithm is derived in an iterative interpolation fashion for the linear case and in a projected gradient descent fashion for the polynomial case, following the idea of the papers cited above.
We then implemented these algorithms in \texttt{sklearn} and we run some empirical tests. Unfortunately the results we obtained are not so promising as the ones obtained in the papers \cite{cortes20092}, \cite{cortes2009learning}. We believe that this is for two main reasons:
\begin{itemize}
\item The selection of the parameters is not easy 
\item There is a trade off between the convergence properties of PGD and the ability of fitting the data of the associated KRR model
\end{itemize}
We believe that with a better selection of parameters we can achieve the same results as in the papers cited above. 
We also tried to look at the generalization bound for these family of kernels. We started from the bound proved in \cite{cortes2010generalization} for the linear case and we tried to get a bound for the polynomial case in a similar fashion. Unfortunately, we were just able to show that such error is larger for polynomial combinations (even if it is of the same order in the number of samples). The bound is obtained by simply embedding the family of polynomial combinations in a larger linear family.
The last thing we looked at is how such learned kernels can perform for SVM. After showing how the optimization problem KRR is a relaxation of the SVM one, we run some empirical tests. The results we obtained are comparable with the ones for the KRR case. 

\nocite{*}
\newpage
\bibliography{reference}
\bibliographystyle{alpha}

\appendix

\section{Appendix}

\subsection{More figures}

\begin{figure}[H]
\includegraphics[width=\textwidth]{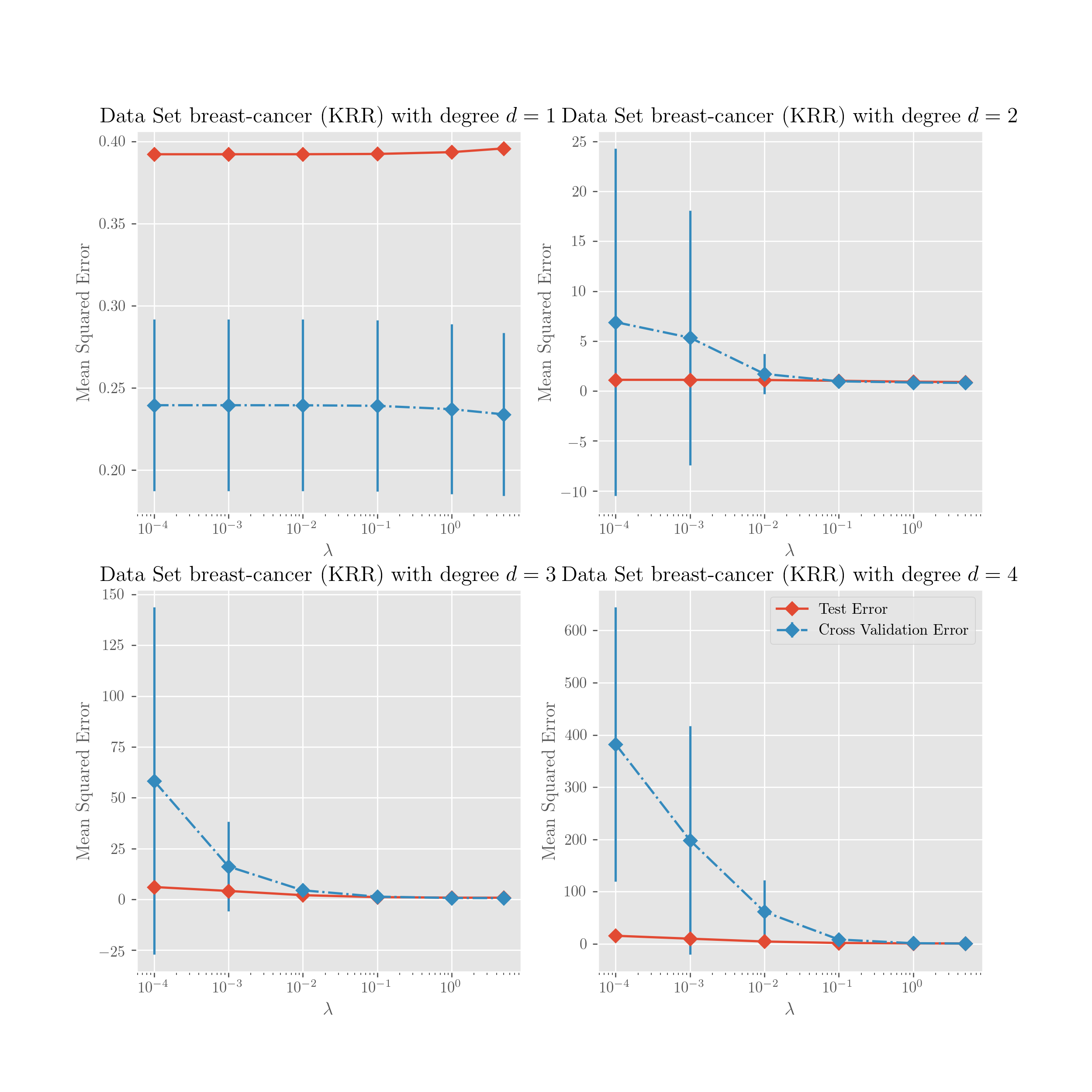}
\caption{PGD results for regression on dataset \texttt{Breast Cancer}.}
\end{figure}


\begin{figure}[H]
\includegraphics[width=\textwidth]{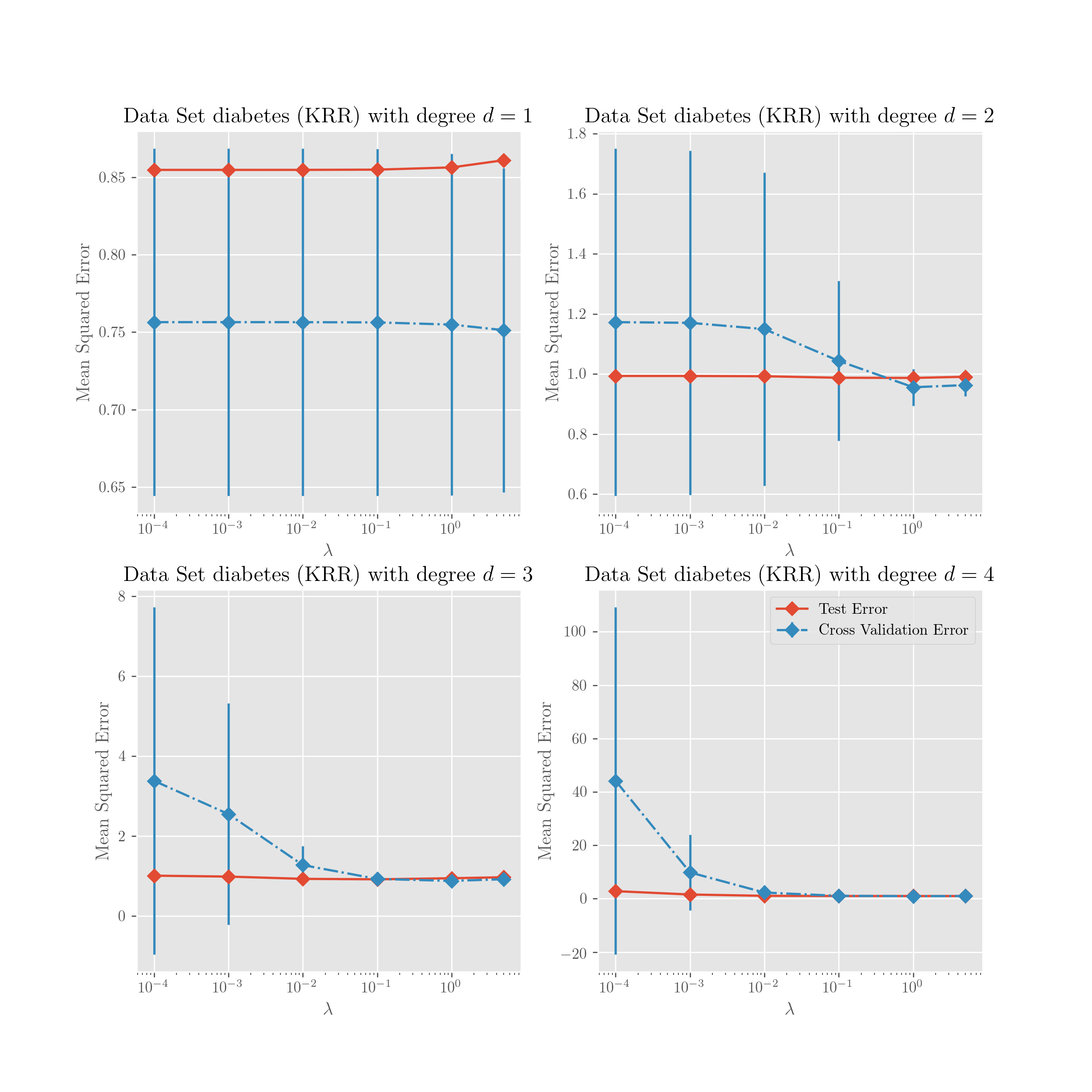}
\caption{PGD results for regression on dataset \texttt{Diabetes}.}
\end{figure}


\begin{figure}[H]
\includegraphics[width=\textwidth]{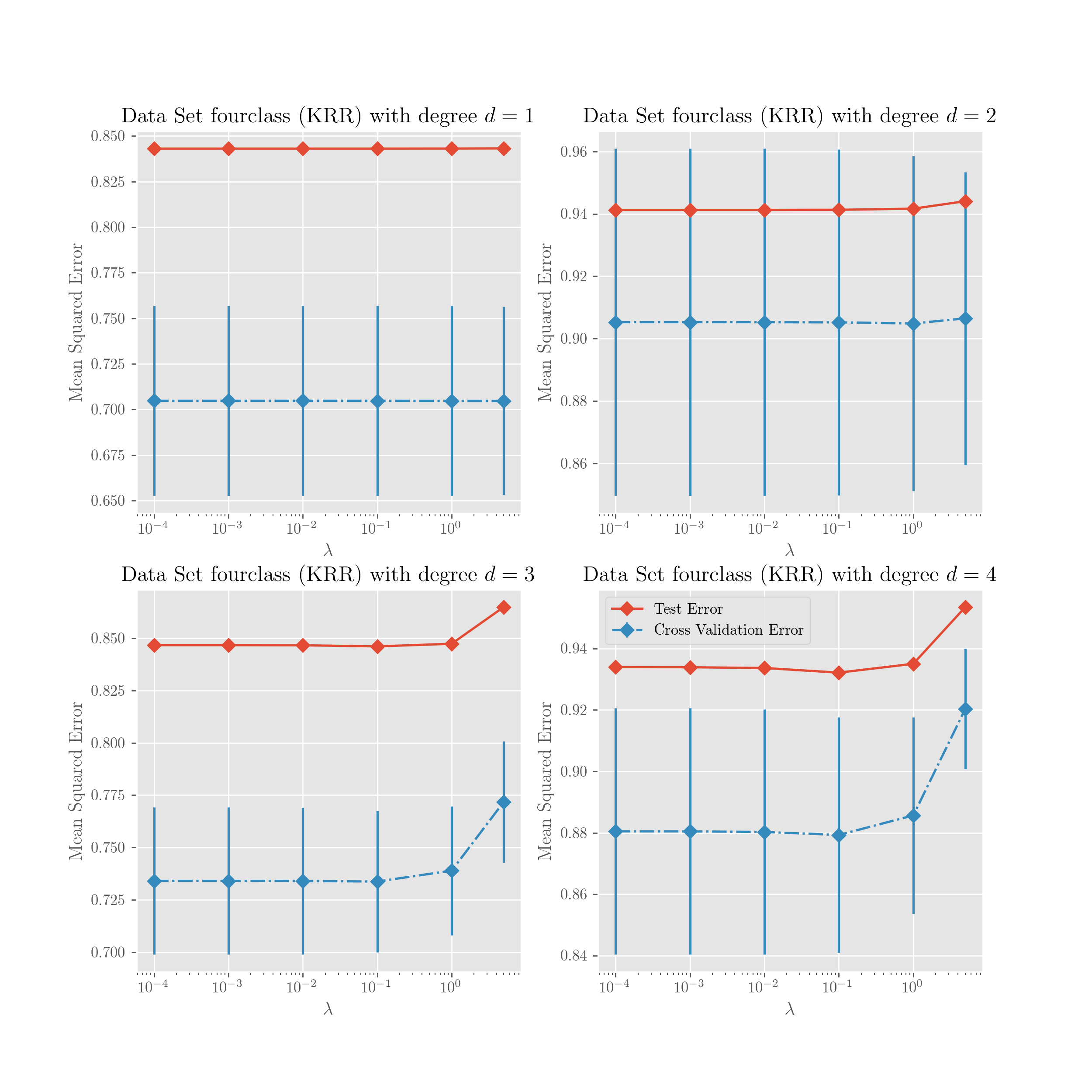}
\caption{PGD results for regression on dataset \texttt{Fourclass}.}
\end{figure}


\begin{figure}[H]
\includegraphics[width=\textwidth]{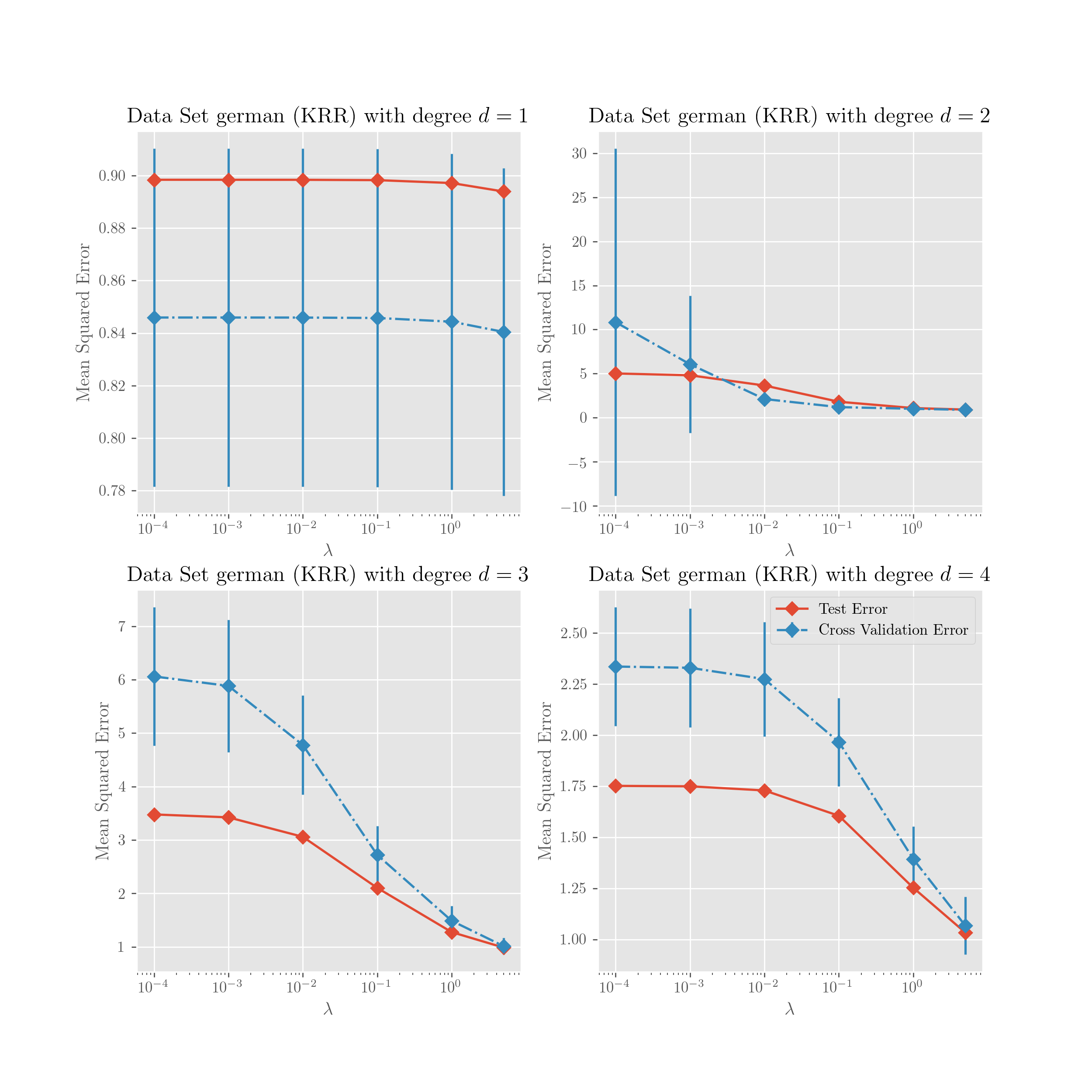}
\caption{PGD results for regression on dataset \texttt{German}.}
\end{figure}


\begin{figure}[H]
\includegraphics[width=\textwidth]{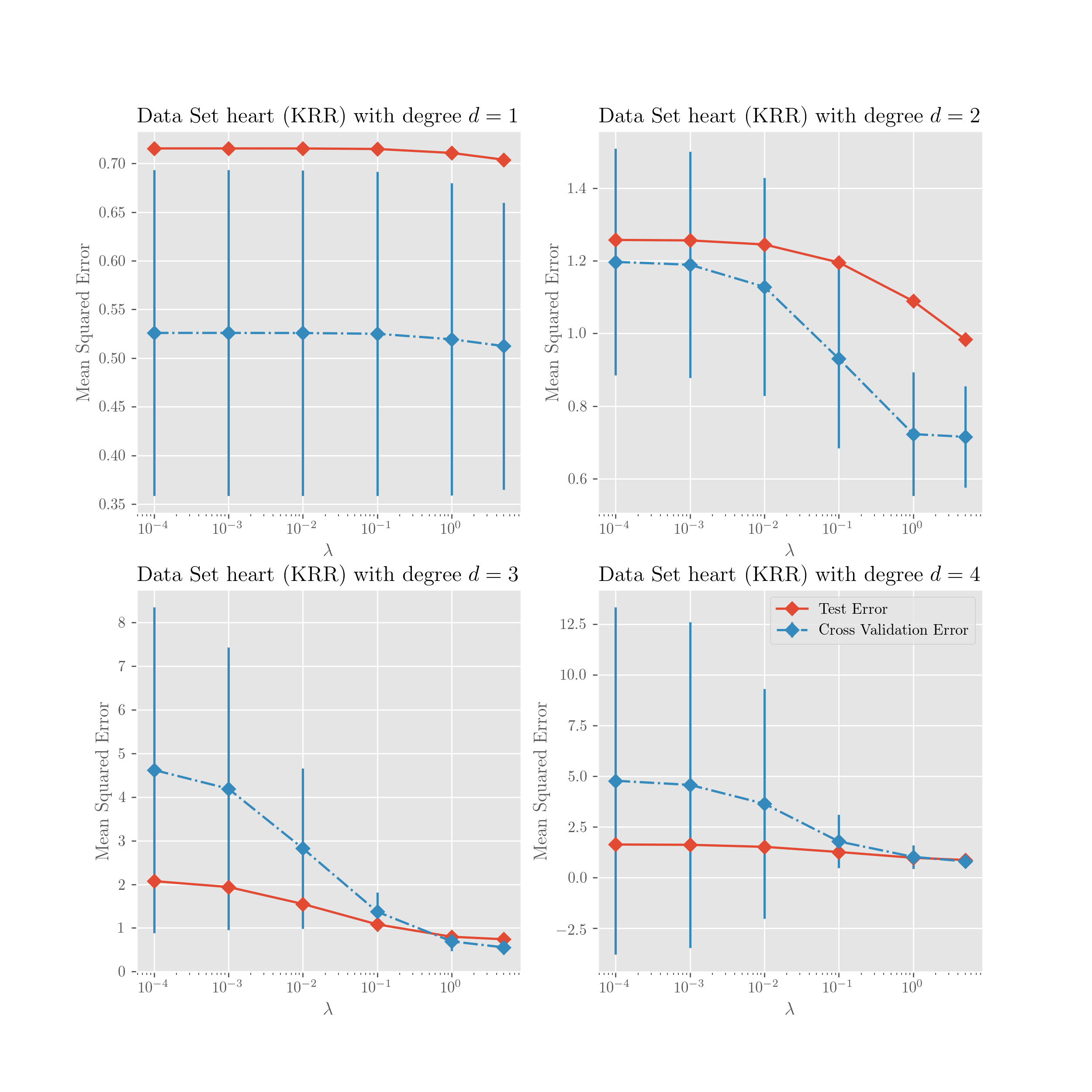}
\caption{PGD results for regression on dataset \texttt{Heart}.}
\end{figure}


\begin{figure}[H]
\includegraphics[width=\textwidth]{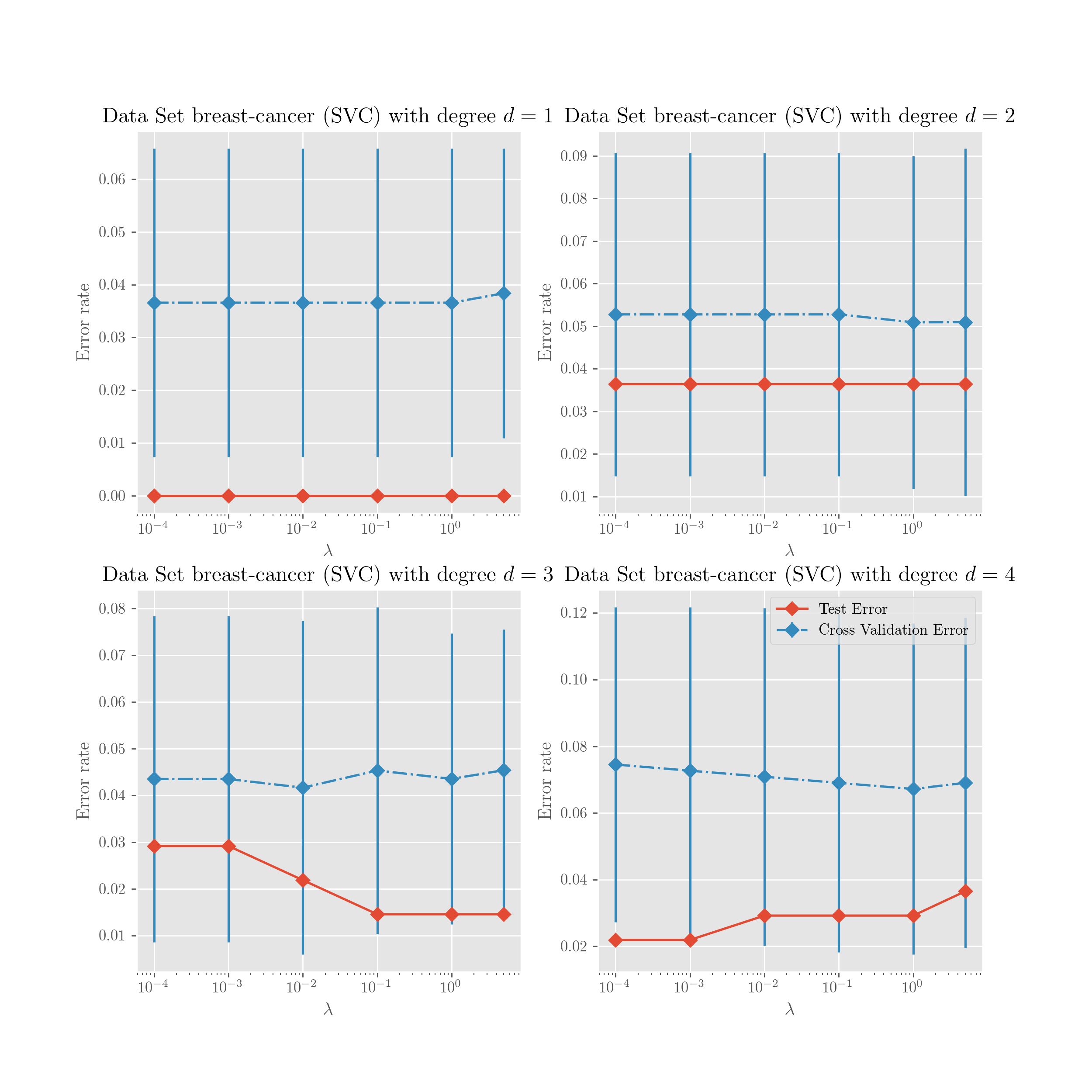}
\caption{PGD results for classification on dataset \texttt{Breast}.}
\end{figure}


\begin{figure}[H]
\includegraphics[width=\textwidth]{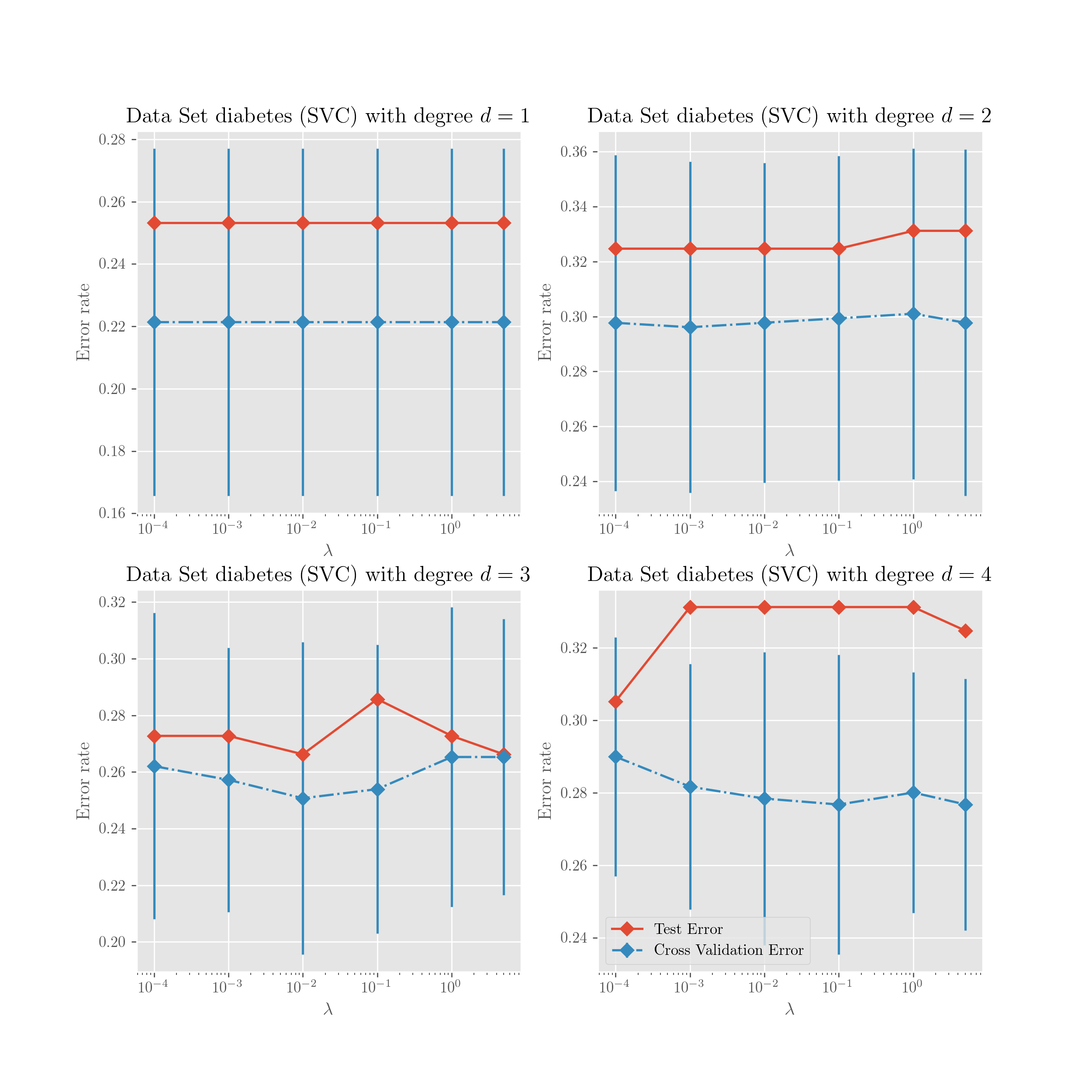}
\caption{PGD results for classification on dataset \texttt{Diabetes}.}
\end{figure}


\begin{figure}[H]
\includegraphics[width=\textwidth]{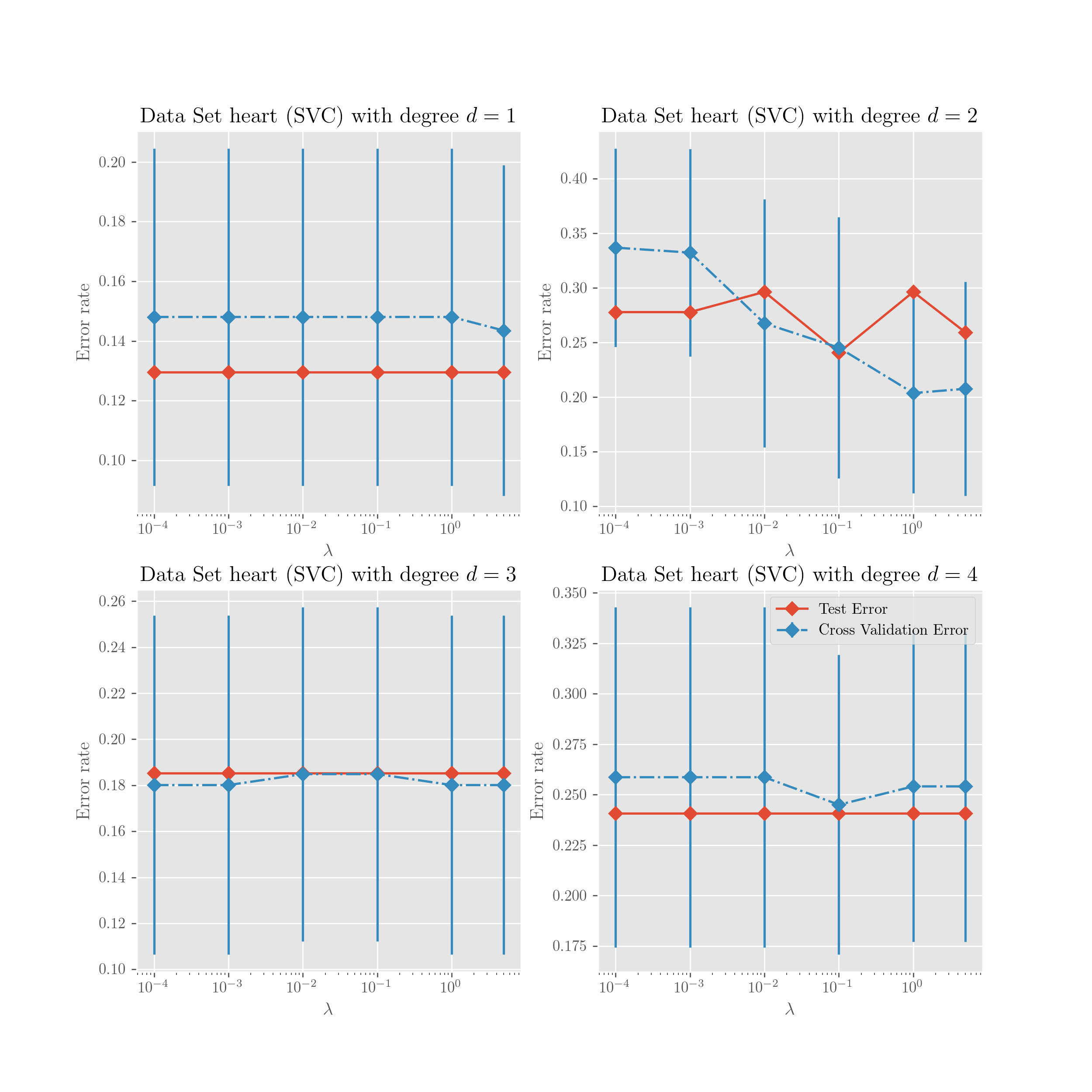}
\caption{PGD results for classification on dataset \texttt{Heart}.}
\end{figure}


\subsection{Proof of Proposition \ref{prop:geo} }\label{app:proof_}

In the section we prove the inequalities:
\begin{equation}\label{RHS:upper}
\bal^T\bK_\bu^{\circ 2}\,\bal \leq B\norm{\bal}^2,
\end{equation}
\begin{equation}\label{LHS:upper}
4\,\bal^T(\bK_\bmu\circ\bK_\bu) (\bK_\bmu^{\circ 2} + \lambda I)^{-1} (\bK_\bmu\circ\bK_\bu)\, \bal \leq \frac{4\norm{\bal}^2}{\lambda}C\,,
\end{equation}
\begin{equation}\label{RHS:lower}
\bal^T\bK_\bu^{\circ 2}\,\bal \geq D\norm{\bal}^2,
\end{equation}
\begin{equation}\label{LHS:lower}
4\,\bal^T(\bK_\bmu\circ\bK_\bu) (\bK_\bmu^{\circ 2} + \lambda I)^{-1} (\bK_\bmu\circ\bK_\bu)\, \bal \geq \frac{4\norm{\bal}^2}{FE}e^{-\lambda / E}\,,
\end{equation}
where 
\begin{align*}
B & = \max_{\bu \st \norm{\bu}=1} \lambda_\mathrm{max}(\bK_\bu^{\circ 2}), \\
C & = \max_{\bu \st\norm{\bu} = 1,\,\bmu \in \mathcal{M}}\norm{(\bK_\bmu \circ \bK_\bu)}^2, \\
D & = \min_{\bu \st \norm{\bu}=1} \lambda_\mathrm{min}(\bK_\bu^{\circ 2}), \\
E & = \max_{\bmu \in\mathcal{M}} \lambda_\mathrm{max}(\bK_\bu^{\circ 2}), \\
F & = \max_{\bu \st\norm{\bu} = 1,\,\bmu \in \mathcal{M}}\norm{(\bK_\bmu \circ \bK_\bu)^{-1}}^2 .
\end{align*}
It holds that (since $\bK_\bu^{\circ 2}$ is symmetric PSD)
\begin{align*}
\bal^T\bK_\bu^{\circ 2}\,\bal & \leq \lambda_\mathrm{max}(\bK_\bu^{\circ 2})\norm{\bal}^2 \leq B \norm{\bal}^2.
\end{align*}
This shows (\ref{RHS:upper}). For the LHS it holds instead:
\begin{align*}
4\,\bal^T(\bK_\bmu\circ\bK_\bu) (\bK_\bmu^{\circ 2} + \lambda I)^{-1} (\bK_\bmu\circ\bK_\bu)\, \bal & \leq 4 \norm {(\bK_\bmu\circ\bK_\bu)\, \bal}^2 \lambda_\mathrm{max}\parr*{(\bK_\bmu^{\circ 2} + \lambda I)^{-1}} \\ 
& \leq 4 \norm{\bK_\bmu \circ \bK_\bu}^2\norm {\bal}^2 \lambda_\mathrm{min}\parr*{\bK_\bmu^{\circ 2} + \lambda I}^{-1} \\
& \leq \frac{4\norm{\bal}^2}{\lambda}C,
\end{align*}
since $\lambda_\mathrm{min}\parr*{\bK_\bmu^{\circ 2} + \lambda I} \geq \lambda$. This shows (\ref{LHS:upper}). The next bounds are proved similarly. First, we have that 
(since $\bK_\bu^{\circ 2}$ is symmetric PSD)
\begin{align*}
\bal^T\bK_\bu^{\circ 2}\,\bal & \geq \lambda_\mathrm{min}(\bK_\bu^{\circ 2})\norm{\bal}^2 \geq D \norm{\bal}^2.
\end{align*}
This shows (\ref{RHS:lower}). Finally, for the LHS it holds:
\begin{multline*}
4\,\bal^T(\bK_\bmu\circ\bK_\bu) (\bK_\bmu^{\circ 2} + \lambda I)^{-1} (\bK_\bmu\circ\bK_\bu)\, \bal \geq \\ \geq 4\norm{\bal}^2 \lambda_\mathrm{min}\parr*{(\bK_\bmu\circ\bK_\bu) (\bK_\bmu^{\circ 2} + \lambda I)^{-1} (\bK_\bmu\circ\bK_\bu)} \\  = 4\norm{\bal}^2 \lambda_\mathrm{max}\parr*{(\bK_\bmu\circ\bK_\bu)^{-1} (\bK_\bmu^{\circ 2} + \lambda I) (\bK_\bmu\circ\bK_\bu)^{-1}}^{-1}.
\end{multline*}
Since 
\begin{multline*}
\lambda_\mathrm{max}\parr*{(\bK_\bmu\circ\bK_\bu)^{-1} (\bK_\bmu^{\circ 2} + \lambda I) (\bK_\bmu\circ\bK_\bu)^{-1}} = \\ = \max_{\bv\st\norm{\bv}=1}\bv^T(\bK_\bmu\circ\bK_\bu)^{-1} (\bK_\bmu^{\circ 2} + \lambda I) (\bK_\bmu\circ\bK_\bu)^{-1}\bv \\
 \leq \norm{(\bK_\bmu\circ\bK_\bu)^{-1}}^2 \max_{\bv\st\norm{\bv}=1}\bv^T (\bK_\bmu^{\circ 2} + \lambda I)\bv \\
 \leq F (\lambda + \lambda_\mathrm{max}(\bK^{\circ 2}_\bmu)) \leq F (\lambda + E) = FE \parr*{1 + \frac{\lambda}{E}},
\end{multline*}
it follows that 
\begin{equation*}
4\,\bal^T(\bK_\bmu\circ\bK_\bu) (\bK_\bmu^{\circ 2} + \lambda I)^{-1} (\bK_\bmu\circ\bK_\bu)\, \bal \geq \frac{4\norm{\bal}^2}{FE}\frac{1}{1 + \frac{\lambda}{E}}.
\end{equation*}
Since $(1+x)e^{-x}\leq 1$ for $x\geq -1$, then it holds
\begin{equation*}
4\,\bal^T(\bK_\bmu\circ\bK_\bu) (\bK_\bmu^{\circ 2} + \lambda I)^{-1} (\bK_\bmu\circ\bK_\bu)\, \bal \geq \frac{4\norm{\bal}^2}{FE}e^{-\lambda / E}.
\end{equation*}
This shows (\ref{LHS:lower}) and thus concludes the proof.

\subsection{Other ideas: manifold optimization}

In \cite{cortes2009learning}, authors provide the PGD for learning quadratic kernels, i.e.
\begin{align}
\bK_\bmu = \parr[\Big]{\sum_{i=1}^p\mu_i\bK_i}^{\circ 2}.
\end{align}
As we have seen above, one geometrical property of problem \ref{lk_krr_linear} is that any minimizer of the function
\begin{equation}
F(\bmu) = \by^T (\bK_\bmu +\lambda I)^{-1}\by     .
\end{equation}
on the manifold
\begin{align}
\mathcal{M} = \bra{\bmu \st \bmu\geq 0,\, \norm{\bmu-\bmu_0} \leq \Lambda}.
\end{align}
can only be achieved on the boundary of the manifold $\partial\mathcal{M}$.
Assuming that $\bmu_0 > 0$ and $\Lambda \leq \|\bmu_0\|_\infty$, then
\begin{align}
\mathcal{M} = \bra{\bmu \st \bmu\geq 0,\, \norm{\bmu-\bmu_0} \leq \Lambda} = \bra{\bmu \st \norm{\bmu-\bmu_0} \leq \Lambda}
\end{align}
forms a smooth submanifold of the vector space $\RR^p$, and we can try to do optimization on manifold $\calM$ with objective function $F(\bmu)$.

We briefly introduce some notation. Let $f$ be a real-valued function from $\calN$ to $\RR$ and $F(\bmu):\calM\to\calN$ be a smooth mapping between two manifolds $\calM$ and $\calN$, and let $\xi_\bmu$ be a tangent vector at a point $\bmu$ of $\calM$, then the mapping $\D F(\bmu)[\xi_\bmu]$ defined by
\begin{align}
\D F(\bmu)[\xi_\bmu]f\defeq\xi_\bmu(f\circ F)
\end{align}
is a tangent vector to $\calN$ at $F(\bmu)$. The tangent vector $DF(\bmu)[\xi_\bmu]$ is realized by $F\circ \gamma$, where $\gamma$ is any curve that realizes $\xi_\bmu$. The mapping
\begin{align}
\D F(\bmu):T_\bmu\calM\to T_{F(\bmu)}\calN, 
\end{align}
(where $T_{F(\bmu)}\calN: \xi\to\D F(\bmu)[\xi]$) is a linear mapping called tangent map. If $\calN=\RR$, we simply have
\begin{align}
\D F(\bmu)[\xi_\bmu] = \xi_\bmu F.
\end{align}

In the case that we are interested in, the $p-1$ dimensional sphere $\SSS^{p-1}$ is naturally embedded in $\RR^p$. If $\bar f$ is a real-valued function in the neighborhood $\calU$ of $\bmu$ in $\RR^p$ and $f$ the restriction of $\bar f$ to $\calU\cap\calM$, then we have tangent vector
\begin{align}
\dot\gamma f = \left. \frac{d}{dt} \bar f(\gamma(t))\right|_{t=0} = \D \bar f(\bmu)[\gamma^\prime(0)].
\end{align}
Let $\gamma: t\mapsto\bmu(t)$ be a curve in the sphere $\SSS^{p-1} + \bmu_0$ through $\hat\bmu$ at $t=0$. Because $\bmu(t)\in\SSS^{p-1} + \bmu_0$ for all $t$, we have
\begin{align}
(\bmu(t)-\bmu_0)^T(\bmu(t)-\bmu_0) = \Lambda,
\end{align}
i.e.
\begin{align}
\dot\bmu(t)^T\bmu(t) + \bmu(t)^T\dot\bmu(t) - \bmu_0^T\dot\bmu(t) - \dot\bmu(t)^T\bmu_0 = 0.
\end{align}
Hence, it is not hard to see that
\begin{align}
    T_{\hat\bmu}(\SSS^{p-1}+\bmu_0) = \ker(\D \bar f(\hat\bmu)) = \left\{\dot\bmu\in\RR^p\st (\hat\bmu - \bmu_0)^T\dot\bmu = 0\right\}.
\end{align}

Given a smooth scalar field $f$ on a Riemannian manifold $\calM$, the gradient of $f$ at $\bmu$ is denoted by $\grad f(\bmu)$, which is defined as the unique element of $T_\bmu\calM$ that satisfies
\begin{align}
\inner{\grad f(\bmu)}{\xi}_\bmu = \D f(\bmu)[\xi],\quad \forall \xi\in T_\bmu\calM.
\end{align}
Let $\calM$ be an embedded submanifold of a Riemannian manifold $\overline \calM$. Every tangent space $T_\bmu\calM$ can be regarded as a subspace of $T_\bmu\overline\calM$ and the orthogonal complement of $T_\bmu\calM$ in $T_\bmu\overline\calM$ is called the normal space to $\calM$ at $\bmu$, i.e.
\begin{align}
(T_\bmu\calM)^\perp = \{\xi\in T_\bmu\overline\calM\st \inner{\xi}{\zeta} = 0,\quad \forall \zeta\in T_\bmu\calM\}.
\end{align}

On the sphere $\SSS^{p-1} + \bmu_0$, considered as Riemannian submanifold of $\RR^p$, the inner product inherited from the standard inner product on $\RR^p$ is given by
\begin{align}
\inner{\xi}{\eta}_\bmu \defeq \xi^T \eta.
\end{align}
Hence, the normal space is
\begin{align}
(T_{\hat\bmu}(\SSS^{p-1} + \bmu_0))^\perp = \{\alpha(\hat\bmu - \bmu_0) \st \alpha\in\RR\},
\end{align}
and the orthogonal projections are given by
\begin{gather}
\mathrm{P}_{\hat\bmu}\xi = (I - \frac{1}{\Lambda^2} (\hat\bmu - \bmu_0)(\hat\bmu - \bmu_0)^T)\xi;\\
\mathrm{P}^\perp_{\hat\bmu}\xi = \frac{1}{\Lambda^2} (\hat\bmu - \bmu_0)(\hat\bmu - \bmu_0)^T\xi.
\end{gather}

As the objective function of learning kernel problem is
\begin{equation}
F(\bmu) = \by^T (\bK_\bmu +\lambda I)^{-1}\by,
\end{equation}
it is easy to see that 
\begin{align}
(\nabla F(\bmu))_k = \frac{\partial F}{\partial \mu_k} = -2 \*\alpha^T \*U_k\*\alpha,
\end{align}
where $\*U_k = (\sum_{r=1}^p \mu_r\bK_r)\circ \bK_k$. We can then obtain $\grad\, F$ by 
\begin{align}
\grad\, F = \mathrm{P}_{\bmu} \nabla F = (I - \frac{1}{\Lambda^2} (\hat\bmu - \bmu_0)(\hat\bmu - \bmu_0)^T) \nabla F.
\end{align}

Let $\calM$ be a Riemannian submanifold of a Riemannian manifold $\overline\calM$ and let $\nabla$ and $\overline\nabla$ denote the Riemannian connections on $\calM$ and $\overline\calM$. Then
\begin{align}\label{eq:second}
    \nabla_{\eta_\bmu}\xi = \mathrm{P}_\bmu \overline\nabla_{\eta_\bmu}\xi
\end{align}
for all $\eta_\bmu\in T_\bmu\calM$. When $\calM$ is a Riemannian submanifold of a Euclidean space, \eqref{eq:second} reads
\begin{align}
    \nabla_{\eta_\bmu}\xi = \mathrm{P}_\bmu(\D\, \xi(\bmu)[\eta_\bmu]).
\end{align}
On the sphere $\SSS^{p-1}+\bmu_0$, viewed as a Riemannian submanifold of the Euclidean space $\RR^p$, the Hessian is given by
\begin{align}
    \Hess\, F(\bmu)[\eta_\bmu] & = \nabla_{\eta_\bmu}\,\grad\, F\\
    & = \left(I - \frac{1}{\Lambda^2} (\hat\bmu - \bmu_0)(\hat\bmu - \bmu_0)^T\right) \D\,\grad F(\bmu)[\eta_\bmu]
\end{align}
With $\grad\, F$ and $\Hess\, F$ in hand, we can perform any local optimization algorithm on the matrix manifold. There is still a need of retraction back onto the manifold after every step of moving on tangent plane, i.e.
\begin{align}
R_\bmu(\xi) = \frac{\bmu+\xi}{\|\bmu+\xi\|}.
\end{align}
This gives a series of alternatives of local methods on learning quadratic kernels. One of them can be trust region method as the algorithm reported
below. Even if we have not tried to implement this, we think it could be interesting to compare its performances against the PGD algorithm previously described.

\begin{algorithm}
\caption{Trust region on matrix manifold}
\label{alg:manifold}
\begin{algorithmic}[1]
\STATE Given $\hat \Delta > 0$, $\Delta_0\in(0,\hat\Delta)$, and $\eta = [0,\frac{1}{4}]$ 

\FOR{$k=0,1,2,\ldots$}
    \STATE $\bmu_k = \argmin f_k + (\grad f)^T_k \bmu + \frac{1}{2} \bmu^T (\Hess f)^T_k\bmu, \quad \|\bmu\|\leq \Delta_k $
    \STATE $\rho_k = \frac{f(x_k) - f(x_k+\bmu_k)}{m_k(0)-m_k(\bmu_k)}$
    \IF{$\rho_k<\frac{1}{4}$}
    \STATE $\Delta_{k+1}=\frac{1}{4}\Delta_k$
    \ELSE
    \IF{$\rho_k>\frac{3}{4}, \|\bmu_k\|=\Delta_k$}
    \STATE $\Delta_{k+1} = \min(2\Delta_k,\hat\Delta)$
    \ELSE
    \STATE $\Delta_{k+1}=\Delta_k$
    \ENDIF
    \ENDIF
    \IF{$\rho_k>\eta$}
    \STATE $x_{k+1} = x_k + \bmu_k$
    \ELSE
    \STATE $x_{k+1} = x_k$
    \ENDIF
    \STATE Retract $x_{k+1}$ onto the manifold
\ENDFOR
\end{algorithmic}
\end{algorithm}

\end{document}